\documentclass[conference]{IEEEtran}
\usepackage{times}
\usepackage{graphicx}

\usepackage[numbers]{natbib}
\usepackage{amsmath}
\usepackage{amssymb}
\usepackage{xcolor}
\usepackage{booktabs}
\usepackage{multirow}
\usepackage{dblfloatfix}

\usepackage{amsthm}
\let\labelindent\relax
\usepackage{enumitem}
\usepackage{needspace}
\usepackage[bookmarks=true]{hyperref}
\usepackage{etoolbox}
\AtBeginEnvironment{thebibliography}{\interlinepenalty=10000\relax}

\usepackage[switch*]{lineno}
\newcommand*\patchAmsMathEnvironmentForLineno[1]{%
  \expandafter\let\csname old#1\expandafter\endcsname\csname #1\endcsname
  \expandafter\let\csname oldend#1\expandafter\endcsname\csname end#1\endcsname
  \renewenvironment{#1}%
    {\linenomath\csname old#1\endcsname}%
    {\csname oldend#1\endcsname\endlinenomath}}
\newcommand*\patchBothAmsMathEnvironmentsForLineno[1]{%
  \patchAmsMathEnvironmentForLineno{#1}%
  \patchAmsMathEnvironmentForLineno{#1*}}
\AtBeginDocument{%
  \patchBothAmsMathEnvironmentsForLineno{equation}%
  \patchBothAmsMathEnvironmentsForLineno{align}%
  \patchBothAmsMathEnvironmentsForLineno{gather}%
  \patchBothAmsMathEnvironmentsForLineno{multline}%
}
\newcommand{\appautoref}[1]{%
  \begingroup
  \renewcommand{\sectionautorefname}{Appendix}%
  \autoref{#1}%
  \endgroup
}

\newtheorem{proposition}{Proposition}
\definecolor{approxred}{RGB}{151,72,45}

\newcommand{\E}{\mathbb{E}}
\newcommand{\Expect}[1]{\mathbb{E}_{#1}}

\begin{document}

\title{DIA: Denoising Intermediate Advantage \\for Diffusion Policy Optimization}

\author{\authorblockN{Arjun Sohal\textsuperscript{1},
Yuchi Zhao\textsuperscript{1,2},
Miroslav Bogdanovic\textsuperscript{1,2,3} and
Al\'an Aspuru-Guzik\textsuperscript{1,2,3,4,5}}

\authorblockA{\textsuperscript{1}University of Toronto, Toronto, ON, Canada}
\authorblockA{\textsuperscript{2}Vector Institute for Artificial Intelligence, Toronto, ON, Canada}
\authorblockA{\textsuperscript{3}Acceleration Consortium, University of Toronto, Toronto, ON, Canada}
\authorblockA{\textsuperscript{4}Canadian Institute for Advanced Research (CIFAR), Toronto, ON, Canada}
\authorblockA{\textsuperscript{5}NVIDIA, Toronto, ON, Canada}}

\maketitle

\begin{abstract}

Diffusion-based robot policies have become widely used in robotic manipulation, where they are typically trained with behavior cloning. However, policies trained purely from demonstrations are limited by the quality and coverage of the available data. Reinforcement learning can further improve the performance of these pretrained policies through interaction. A common approach is to use policy-gradient methods that formulate diffusion-policy fine-tuning as an outer environment MDP together with an inner denoising MDP. However, existing methods typically assign the same environment-level credit to all denoising steps used to construct an action chunk, without distinguishing which intermediate decisions contributed most to the final return. We introduce \emph{Denoising Intermediate Advantage} (DIA), a policy-gradient method that learns a value function over partially denoised actions and uses it to construct a denoising-level advantage for each step of the generative process. DIA combines this inner credit signal with the standard environment-level PPO advantage, providing state-dependent credit throughout the denoising chain. Across Robomimic, FurnitureBench, Franka Kitchen, and D3IL, DIA consistently improves final performance over existing diffusion-policy fine-tuning methods. Beyond final reward, DIA reaches successful states more efficiently and can shift farther from the pretrained behavior distribution, enabling it to discover more effective and efficient task-level strategies and subtask sequences that baseline methods fail to reach. 

\end{abstract}

\IEEEpeerreviewmaketitle

\section{Introduction}

Diffusion-based policies have become a common approach for behavior-cloning (BC) pre-training in robotic manipulation, as their iterative denoising process can represent the multimodal action distributions present in demonstration data~\cite{chi2023diffusion}. However, BC policies remain constrained by the quality and coverage of the demonstrations on which they are trained. To further improve policy performance, recent work has fine-tuned pre-trained diffusion policies with reinforcement learning (RL), allowing the policy to improve through interaction with the environment. In particular, policy-gradient methods treat the denoising process as a sequential decision process, enabling each denoising transition to be optimized using task reward~\cite{ren2025diffusion}.

\begin{figure}[tb]
\centering
\includegraphics[width=1.04\columnwidth]{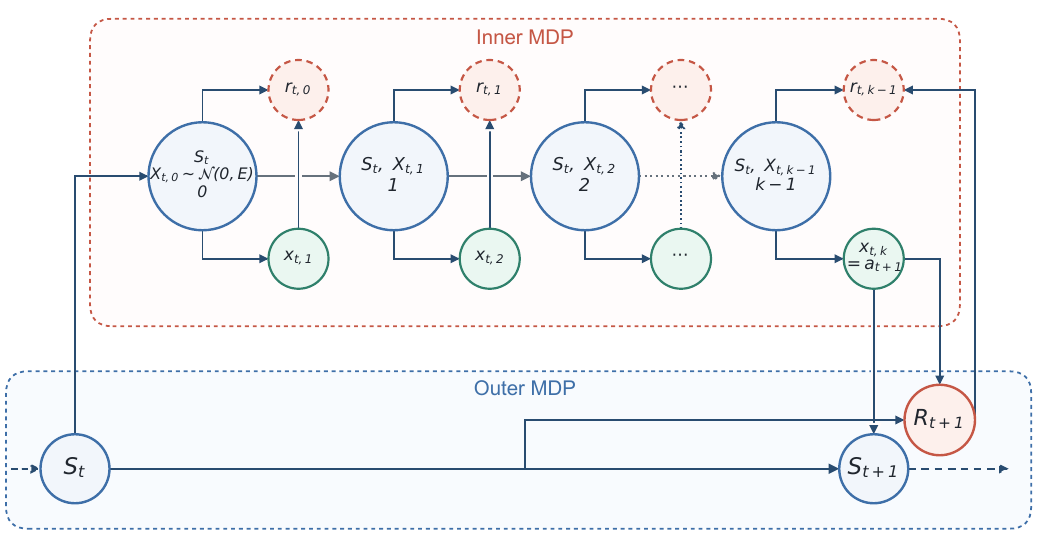}
\caption{Two-level MDP formulation for fine-tuning a diffusion policy. At each outer state $S_t$, the inner MDP generates the action $a_{t+1}$ through iterative denoising from $X_{t,0}$, where the inner reward is zero, i.e., $r_{t,k}=0$. The resulting action induces the transition to $S_{t+1}$ and produces the outer reward $R_{t+1}$, which provides feedback for policy optimization.}
\label{fig:mdp}
\end{figure}

Unlike single-step policies, diffusion policies generate actions through a sequence of stochastic denoising steps, making credit assignment and stable policy-gradient optimization more challenging during RL fine-tuning. To tackle this, one common approach is to formulated through a two-level Markov decision process (MDP) as shown in Figure~\ref{fig:mdp}: an outer MDP over environment states and an inner generative MDP over the denoising process that constructs each action chunk \cite{Zhao-RSS-23}. However, while these methods optimize the individual transitions of the inner MDP, they define value only at the environment level. The resulting environment-level advantage is therefore used to update all denoising decisions that produced the executed action, with DPPO additionally applying a fixed weighting based on denoising depth~\cite{ren2025diffusion} as a heuristic. This weighting varies with the position in the denoising chain, but not with the realized intermediate denoising state. We argue that an intermediate denoised action contains information about the action toward which the policy is committing, and this information changes throughout generation. Treating all intermediate decisions using only an environment-level signal therefore leaves the structure of the denoising trajectory itself unused for credit assignment.

We propose \emph{Denoising Intermediate Advantage} (DIA), a policy-gradient method for fine-tuning diffusion-based robot policies. Specifically, we introduce an inner advantage derived from a value function over partially denoised actions, allowing the policy update to depend on the realized intermediate state of the denoising process. Across Robomimic, FurnitureBench, Franka Kitchen, and D3IL, DIA consistently improves task performance over existing diffusion-policy fine-tuning methods, with the largest gains on more challenging long-horizon tasks. These improvements hold for both state- and image-based policies. Beyond reward and success rate, DIA reaches successful states more efficiently, shifts further from the pre-trained behavior distribution, and composes demonstrated subtask behaviors into successful task-level sequences that baseline methods fail to discover. Our main contributions are: 
\begin{itemize}
    \item We introduce DIA, a policy-gradient method that learns values over
    intermediate denoising states and uses them to construct state-dependent
    credit for individual denoising decisions.

    \item We derive the corresponding denoising-level advantage and show how it
    can be combined with the standard environment-level policy-gradient
    estimator for diffusion-policy fine-tuning.

    \item Across four benchmarks, we show that DIA improves reward and
    task success while also producing more efficient executions and discovering
    successful behaviors further from pre-trained policies.
\end{itemize}

\section{Related Work}
\label{sec:related}

\subsection{Policy gradients through the generative chain}
These methods apply the policy gradient theorem to the optimization of generative
policies, treating each denoising or integration step as a decision of an extended MDP. DDPO~\cite{black2024training}
casts the $K$ denoising steps of a diffusion model as a multi-step MDP and fine-tunes it
with policy gradients against a non-differentiable reward, and DPOK~\cite{fan2023reinforcement}
does the same with a KL penalty toward the pre-trained model; both target text-to-image
generation, where one scalar is assigned to the finished image.
Flow-GRPO~\cite{liu2025flowgrpo} converts the deterministic generation ODE into an
equivalent stochastic differential equation so that flow models can be optimized with
group-relative policy optimization, replacing a learned value with a baseline computed
across a group of samples.

For control, DPPO~\cite{ren2025diffusion} nests the denoising chain inside the environment
MDP. Every denoising transition is Gaussian, so its log-likelihood is available in closed
form and PPO~\cite{schulman2017proximalpolicyoptimizationalgorithms} applies through the whole chain.
ReinFlow~\cite{zhang2025reinflowfinetuningflowmatching} extends the same treatment to flow policies, whose
generation map is deterministic and whose per-step likelihood is therefore degenerate, by
learning a noise-injection network that turns each integration step into a Gaussian
transition. Flow Policy Gradients~\cite{yi2026flowpolicygradientsrobot} instead trains flow policies while
bypassing explicit likelihoods for the flow map.
ResiP~\cite{11127442} freezes the diffusion policy entirely and trains a
Gaussian residual on top of its actions with PPO, applying the policy gradient beside
the generative chain rather than through it.

These methods share the same source of learning signal. The value function is defined
over environment states alone; the resulting advantage is estimated at the environment
level with GAE~\cite{schulman2018highdimensionalcontinuouscontrolusing} and is then applied to every denoising decision of
that environment step, in DPPO up to a fixed geometric depth schedule that attenuates the
early decisions rather than in proportion to their contribution.
In the image-generation setting a single terminal reward is shared by the entire chain.

\subsection{$Q$-based fine-tuning of diffusion policies}
These methods primarily function through a learned Q critic, and using that critic to guide the generation of actions towards higher value ones.
DQL~\cite{wang2023diffusion} backpropagates a learned $Q$ through the sampler,
QSM~\cite{3692070.3693742} matches the policy score to $\nabla_a Q$,
DIPO~\cite{yang2023policyrepresentationdiffusionprobability} ascends $Q$ in action space and refits the diffusion model to
the improved actions, and IDQL~\cite{hansenestruch2023idql} draws candidate actions and
reweights them by an implicit value.
The same recipe carries over to flow policies: FQL~\cite{pmlr-v267-park25f} trains a one-step
policy distilled from a flow model against a learned $Q$, rather than guiding the
iterative flow directly, which avoids backpropagating through the integration;
QC~\cite{li2025reinforcement} learns values over multi-step action chunks.

Q-Flow~\cite{doo2026qflow} assigns a value to the intermediate generative states
rather than to the completed action alone. It exploits the deterministic flow map to
propagate a terminal trajectory value back to the intermediate latents along the
policy-induced flow, in place of differentiating through the numerical solver, and uses
that value to guide generation toward higher-value actions.

These methods assign credit to actions through a learned $Q$, but the generative
process that produced the action is treated as a fixed sampler: the learning signal
reaches it only through the sampled endpoint, or through differentiation of the sampler
as a whole, and never as credit attributed to an individual generative decision.

\section{Preliminaries}
\label{sec:prelim}

\subsection{Policy-gradient reinforcement learning}
We consider a discounted MDP with state $s_t$, action
$a_t$, reward $r_t$, and discount $\gamma$, controlled by a stochastic policy
$\pi_\theta$. The return from step $t$ is $g_t=\sum_{j\geq0}\gamma^j r_{t+j}$,
the objective is $J(\theta)=\mathbb{E}\bigl[\sum_{t\geq0}\gamma^t r_t\bigr]$,
and the state and action values are $V(s_t)=\mathbb{E}[g_t\mid s_t]$ and
$Q(s_t,a_t)=\mathbb{E}[g_t\mid s_t,a_t]$.
The \emph{policy gradient theorem}~\cite{sutton1999policygradient} writes the
gradient of the objective as an expectation of per-decision score terms weighted by
the return,
\begin{equation}
\nabla_\theta J(\theta) = \mathbb{E}\Bigl[\textstyle\sum_{t\geq0}\gamma^t\,
   \nabla_\theta \log \pi_\theta(a_t\mid s_t)\, g_t\Bigr].
\label{eq:pg}
\end{equation}

Actor--critic methods weight each score term by the \emph{advantage}
$A(s_t,a_t)=Q(s_t,a_t)-V(s_t)$ rather than by the return $g_t$, the subtracted
$V(s_t)$ acting as a variance-reducing \emph{baseline}.
Generalized advantage estimation~\cite{schulman2018highdimensionalcontinuouscontrolusing} (GAE) estimates $A$ from
a learned critic $\widehat V$ as an exponentially weighted sum of
temporal-difference residuals
$\delta_t = r_t + \gamma\widehat V(s_{t+1}) - \widehat V(s_t)$,
\begin{equation}
\widehat A_t^{\lambda} = \sum_{\ell\geq0}(\gamma\lambda)^{\ell}\delta_{t+\ell}
  = \delta_t + \gamma\lambda\,\widehat A_{t+1}^{\lambda},
\label{eq:gae}
\end{equation}
with $\lambda$ trading the critic's bias at $\lambda=0$ against rollout variance
at $\lambda=1$.

\subsection{Diffusion policies}
A diffusion model~\cite{3495724.3496298} draws a sample by starting from
Gaussian noise and applying a learned sequence of $K$ stochastic transitions,
each of which slightly denoises the current latent. Each transition is Gaussian,
with a mean produced by the denoising network and a variance fixed by the
sampler, so its log-density is available in closed form even though the
log-density of the endpoint marginal is not. A diffusion
policy~\cite{chi2023diffusion} uses this process as a continuous control policy:
conditioned on an observation $s$, the chain denoises noise into an action, or a
short chunk of actions, so that $\pi_\theta(a\mid s)$ is the marginal of the
conditioned chain over its endpoint. Representing the action distribution this
way lets a single policy capture the multimodal behavior found in expert
demonstrations, which unimodal Gaussian policies cannot.

\begin{figure*}[!tb]
    \centering
    \includegraphics[width=0.9\textwidth]{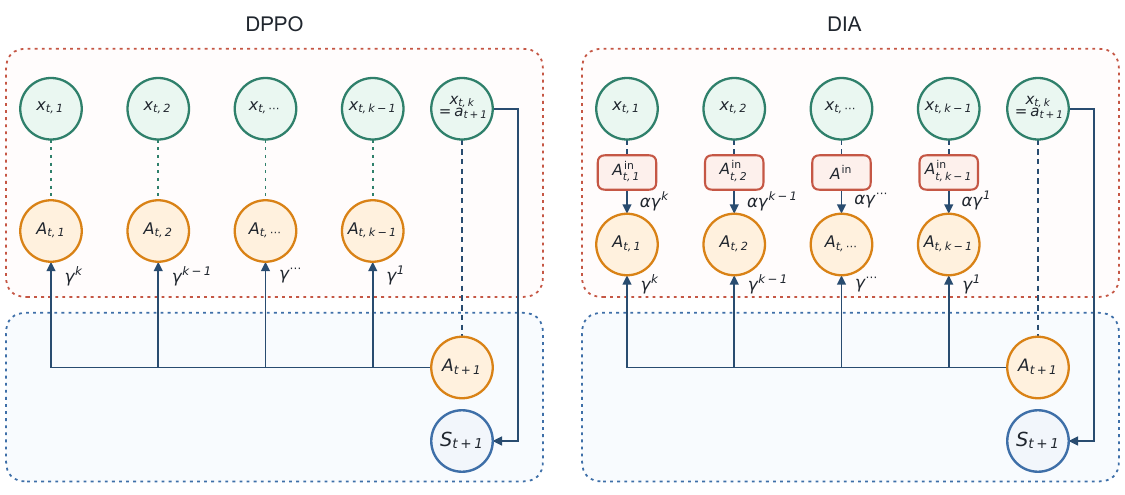}
   \caption{Comparison of denoising-step credit assignment in DPPO and DIA. DPPO applies a shared environment-level advantage with fixed depth-dependent weighting. DIA uses a learned inner value function, bootstrapped by the terminal action value $\widehat Q(s_t,a_t)$, to compute a state-dependent inner advantage $\widehat A_{t,k}^{\mathrm{in}}$, which is combined with the outer advantage using coefficient $\alpha$.}
    \label{fig:dppo_dia}
\end{figure*}

\section{Denoising Intermediate Advantage}
\label{sec:method}

We propose DIA, a policy gradient method for diffusion-based policies.
DIA treats the denoising chain as the sequence of decisions being optimized, and learns a value function along it.
It fits a critic to partially denoised actions, forms a per-decision advantage that bootstraps from the value of the completed action, and mixes that estimator with the standard environment-return estimator under a single coefficient.
The overall mechanism is shown in Figure~\ref{fig:dppo_dia}.

\subsection{The denoising process is the primary MDP}
\label{sec:mdp}

At environment time \(t\), a diffusion policy constructs one environment action through \(K\) stochastic denoising steps.
We index these steps in the order in which they occur:
\begin{align}
    x_{t,0} &\sim \mathcal{N}(0, I), \label{eq:noise}\\
    x_{t,k+1} &\sim \pi_{\theta,k}
        (\,\cdot\mid s_t,x_{t,k})
\end{align}
for \(k=0,\ldots,K-1\), where \(x_{t,0}\) is the initial random noise and \(a_t = x_{t,K}\) the fully denoised action that gets executed in the environment.

The state of the flattened process at denoising decision \(k\) of environment
time \(t\) is the pair \((s_t,x_{t,k})\) together with the denoising index \(k\), and the
decision taken there is the next latent \(x_{t,k+1}\).
We keep the denoising index explicit throughout; equivalently, one may
absorb \(k\) into an augmented state
\((s_t,x_{t,k},k)\), which makes the flattened process time homogeneous: the
same pair \((s,x)\) generally has a different transition kernel at a
different noise level.

The environment reward enters at the denoising step that produces the final action $a_t = x_{t, K}$,
at every other transition there is 0 reward:
\begin{equation}
    r_{t,k} =
    \begin{cases}
        R(s_t,a_t), & k = K-1,\\
        0,          & k < K-1.
    \end{cases}
    \label{eq:flat-reward}
\end{equation}
Additionally, at denoising step $K - 1$, we transition to the next denoising chain and outer environment state:
\begin{equation}
    (s_t,x_{t,k}) \;\xrightarrow{\;x_{t,k+1}\;}\;
    \begin{cases}
        (s_t,x_{t,k+1}),     & k < K-1,\\
        (s_{t+1},x_{t+1,0}), & k = K-1,
    \end{cases}
    \label{eq:flat-transition}
\end{equation}
where \(s_{t+1}\sim P(\,\cdot\mid s_t,a_t)\) comes from the environment transition
kernel and the next chain restarts from fresh noise \(x_{t+1,0}\sim\mathcal{N}(0,I)\).

\subsection{Exact policy gradient on the flattened MDP}
\label{sec:exact-pg}

We have the discounted return for a trajectory and the objective as:
\begin{equation}
G(\tau):=\sum_{t\geq0}\gamma^t r_t,
\qquad
J(\theta):=\E_{\tau\sim p_\theta}[G(\tau)].
\label{eq:trajectory-return}
\end{equation}

The policy gradient is:
\begin{align}
\nabla_\theta J(\theta)
=\E\!\left[
G(\tau)\nabla_\theta\log p_\theta(\tau)
\right]
\nonumber\\[-2pt]
\end{align}

We can express the PDF of a sampled trajectory, \(p_\theta(\tau)\), as:
\begin{equation}
\begin{alignedat}{2}
p_\theta(\tau)
&=P_0(s_0)\prod_{t\geq0}\Bigl[&&
\mathcal \rho(x_{t,0})
\prod_{k=0}^{K-1}
\pi_{\theta,k}(x_{t,k+1}\mid s_t,x_{t,k})
\\[-2pt]
&&&P(s_{t+1}\mid s_t,x_{t, K})\Bigr].
\end{alignedat}
\label{eq:trajectory-density}
\end{equation}

Differentiating with respect to \(\theta\) and removing terms that do not depend on it gives us:
\begin{equation}
\nabla_\theta\log p_\theta(\tau)
=
\sum_{t\geq0}\sum_{k=0}^{K-1}
\nabla_\theta\log\pi_{\theta,k}
(x_{t,k+1}\mid s_t,x_{t,k}).
\label{eq:trajectory-score}
\end{equation}

Substituting \(H_{t,k}
:=
\nabla_\theta\log\pi_{\theta,k}
(x_{t,k+1}\mid s_t,x_{t,k})
\label{eq:score}\), we can write it as:

\begin{equation}
\nabla_\theta\log p_\theta(\tau)
=
\sum_{t\geq0}\sum_{k=0}^{K-1}H_{t,k}.
\end{equation}

The policy gradient is then:
\begin{align}
\nabla_\theta J(\theta)
&=\E\!\left[
G(\tau)\nabla_\theta\log p_\theta(\tau)
\right]
\nonumber
=\sum_{t\geq0}\sum_{k=0}^{K-1}
\E[H_{t,k}G(\tau)]
\nonumber\\
&=\sum_{t\geq0}\gamma^t\sum_{k=0}^{K-1}
\E[H_{t,k}g_t]
=\E\!\left[
\sum_{t\geq0}\gamma^t
\sum_{k=0}^{K-1}H_{t,k}g_t
\right]
\label{eq:exact-pg}
\end{align}

Next, even though, unlike standard policy gradient score function, \(H_{t,k}\) depends on the internal denoising path, it still holds that (full proof in \appautoref{app:policy-grad-deriv}):
\begin{equation}
\nabla_\theta J(\theta)
=
\E\!\left[
\sum_{t\geq0}\gamma^t
\sum_{k=0}^{K-1}
H_{t,k}Q(s_t,a_t)
\right]
\label{eq:pg-Q}
\end{equation}

Subtracting the value function baseline we can further transform that to (full proof in \appautoref{app:policy-grad-deriv}):
\begin{align}
\nabla_\theta J(\theta)
&=
\E\!\left[
\sum_{t\geq0}\gamma^t
\sum_{k=0}^{K-1}
H_{t,k}\bigl(Q(s_t,a_t)-V(s_t)\bigr)
\right]
\nonumber\\
&=
\E\!\left[
\sum_{t\geq0}\gamma^t
\sum_{k=0}^{K-1}
H_{t,k}A(s_t,a_t)
\right]
\label{eq:pg-A}
\end{align}

\subsection{Value and advantage on partially denoised actions}

We can define the value along the internal denoising chain as:
\begin{align}
V_k(s_t,x_{t,k})
&:=\E[g_t\mid s_t,x_{t,k}]
\nonumber\\
&=\E[Q(s_t,a_t)\mid s_t,x_{t,k}].
\label{eq:Vk-def}
\end{align}

With terminal value \(V_K(s,a)\) being equal to the Q-function for the fully denoised action:
\begin{equation}
V_K(s_t,x_{t,K})
=
Q(s_t,a_t),
\qquad
x_{t,K}=a_t.
\label{eq:Vk-terminal}
\end{equation}

Note, that while all states along a single denoising chain have the same return, the value is not constant along it. Different denoising steps assign different probabilities to realized actions, therefore resulting in different expected returns and therefore different values.

We can represent the policy gradient in terms of inner, denoising chain value as:
\begin{align}
\nabla_\theta J(\theta)
&=\E\!\Biggl[
\sum_{t\geq0}\gamma^t
\sum_{k=0}^{K-1}H_{t,k}
\bigl(Q(s_t,a_t)-V_k(s_t,x_{t,k})\bigr)
\Biggr]
\label{eq:pg-inner-value}
\end{align}

Similarly to \(V_k\), we define advantage along the denoising chain as:
\begin{equation}
A_k(s_t,x_{t,k},x_{t,k+1}):=V_{k+1}(s_t,x_{t,k+1})-V_k(s_t,x_{t,k})
\label{eq:exact-inner-adv}
\end{equation}

We can show that the policy gradient formulation from \autoref{eq:pg-inner-value} can further be transformed to be in terms of this inner denoising chain advantage, finally resulting in the policy gradient formula we utilize in our approach (full proof in \appautoref{app:policy-grad-deriv}):
\begin{align}
\nabla_\theta J(\theta)
&=\E\!\Biggl[
\sum_{t\geq0}\gamma^t
\sum_{k=0}^{K-1}H_{t,k}
A_k(s_t,x_{t,k},x_{t,k+1})
\Biggr]
\label{eq:pg-inner-advantage}
\end{align}

\subsection{Estimating the inner advantage}
\label{sec:inner-estimation}

For a practical policy gradient algorithm we now need to estimate the inner advantage \(A_k\). We use GAE within the denoising chain to estimate it.

\begin{equation}
\widehat{\delta}_{t,k}
=
\widehat V_{k+1}(s_t,x_{t,k+1})
-
\widehat V_k(s_t,x_{t,k})
\label{eq:inner-delta}
\end{equation}

\begin{equation}
\widehat A_{t,k}^{\mathrm{in}}
=
\sum_{\ell=0}^{K-1-k}
\lambda_{\mathrm{in}}^\ell
\widehat{\delta}_{t,k+\ell}
\label{eq:inner-gae}
\end{equation}

\begin{figure*}[!t]
\centering
\includegraphics[width=\textwidth]{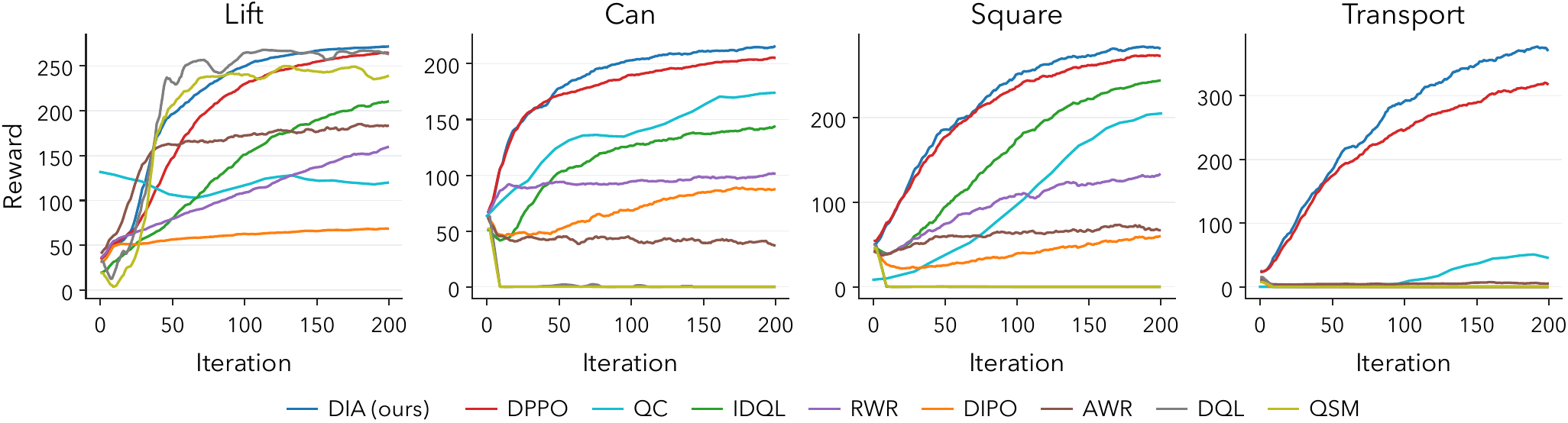}
\caption{Fine-tuning return on robomimic, mean over $5$ seeds.}
\label{fig:results_ci}
\end{figure*}
\begin{table*}[!t]
\centering
\caption{
Final performance on Robomimic. Each entry reports return, with success rate
in parentheses where applicable. Best return per row is shown in bold.
}
\label{tab:robomimic_results}
\footnotesize
\setlength{\tabcolsep}{1.5pt}
\begin{tabular}{lcccccccccc}
\toprule
Task & IDQL & DQL & QSM & RWR & DIPO & AWR & QC & ResiP & DPPO & DIA (ours) \\
\midrule
\textsc{Lift}
& $242.8$ ($1.00$)
& $268.8$ ($1.00$)
& $245.4$ ($0.99$)
& $163.2$ ($0.98$)
& $72.4$ ($0.99$)
& $197.6$ ($0.88$)
& $123.4$ ($0.98$)
& $133.5$ ($0.99$)
& $265.9$ ($1.00$)
& $\mathbf{273.6}$ ($1.00$) \\

\textsc{Can}
& $175.0$ ($1.00$)
& $0.0$ ($0.00$)
& $0.1$ ($0.00$)
& $108.4$ ($0.89$)
& $105.0$ ($0.87$)
& $44.3$ ($0.27$)
& $176.8$ ($0.86$)
& $160.5$ ($0.96$)
& $210.2$ ($1.00$)
& $\mathbf{221.0}$ ($0.98$) \\

\textsc{Square}
& $274.6$ ($0.99$)
& $0.0$ ($0.00$)
& $0.0$ ($0.00$)
& $150.7$ ($0.86$)
& $87.1$ ($0.51$)
& $85.5$ ($0.39$)
& $204.8$ ($0.70$)
& $159.3$ ($0.76$)
& $287.2$ ($0.98$)
& $\mathbf{302.3}$ ($0.98$) \\

\textsc{Transport}
& $0.0$ ($0.00$)
& $0.0$ ($0.00$)
& $0.0$ ($0.00$)
& $0.0$ ($0.00$)
& $0.0$ ($0.00$)
& $7.7$ ($0.07$)
& $46.0$ ($0.11$)
& $114.9$ ($0.50$)
& $356.1$ ($0.90$)
& $\mathbf{438.5}$ ($0.91$) \\

\textsc{Square-Pixel}
& -- & -- & -- & -- & -- & -- & -- & --
& $249.6$ ($0.95$)
& $\mathbf{269.7}$ ($0.95$) \\

\textsc{Can-Pixel}
& -- & -- & -- & -- & -- & -- & -- & --
& $195.2$ ($0.99$)
& $\mathbf{211.3}$ ($0.98$) \\
\bottomrule
\end{tabular}
\end{table*}
\begin{table}[!t]
\centering
\caption{
Final performance, reward (SR) on FurnitureBench, Franka Kitchen, and D3IL.
Best performance per row is shown in bold.
}
\label{tab:other_results}
\footnotesize
\setlength{\tabcolsep}{5pt}
\begin{tabular}{llcc}
\toprule
Benchmark & Task & DPPO & DIA (ours) \\
\midrule

\multirow{2}{*}{FurnitureBench}
& \textsc{One-Leg}
& $203.2$ ($0.79$)
& $\mathbf{227.1}$ ($\mathbf{0.82}$) \\

& \textsc{Lamp}
& $447.0$ ($0.35$)
& $\mathbf{538.4}$ ($\mathbf{0.47}$) \\

\midrule

\multirow{3}{*}{Franka Kitchen}
& \textsc{Complete}
& $\mathbf{3.99}$
& $3.97$ \\

& \textsc{Partial}
& $3.30$
& $\mathbf{3.79}$ \\

& \textsc{Mixed}
& $3.00$
& $\mathbf{3.71}$ \\

\midrule

\multirow{3}{*}{D3IL \textsc{Avoid}}
& \textsc{M1}
& $94.5$
& $\mathbf{99.4}$ \\

& \textsc{M2}
& $93.6$
& $\mathbf{98.9}$ \\

& \textsc{M3}
& $94.6$
& $\mathbf{96.9}$ \\

\bottomrule
\end{tabular}
\end{table}

From \autoref{eq:Vk-def} we have that \(V_k\) is equal to the expected Q-value of the environment MDP for the current state and fully denoised action. We train \(V_k\) using a mean-square error loss with respect to this environment Q-value:

\begin{equation}
\mathcal L_V
=
\E\!\left[
\sum_{k=0}^{K-1}
\left(
\widehat V_k(s_t,x_{t,k})
-
\widehat Q(s_t,a_t)
\right)^2
\right].
\label{eq:V-loss}
\end{equation}

What is left now is to estimate the environment Q-value itself. We learn the estimator for it using TD-learning. We use a mean of an ensemble of independently initialized Q functions for the estimate to reduce noise:

\begin{equation}
\widehat Q(s_t,a_t)
=
\frac{1}{N}
\sum_{n=1}^{N}
\widehat Q^{(n)}(s_t,a_t).
\label{eq:Q-ensemble}
\end{equation}

\subsection{Advantage mixing and scale matching}
\label{sec:fusion}

In addition to the inner denoising chain advantage \(A_k\) we can estimate the normal environment level advantage \(A\) using GAE at this outer level as well:
\begin{align}
\widehat\delta_t^{\mathrm{out}}
&=
r_t+\gamma\widehat V(s_{t+1})-\widehat V(s_t),
\label{eq:outer-delta}\\
\widehat A_t^{\mathrm{out}}
&=
\sum_{\ell\geq0}
(\gamma\lambda)^\ell
\widehat\delta_{t+\ell}^{\mathrm{out}}.
\label{eq:outer-gae}
\end{align}

These two advantage formulations provide complimentary signals. We utilize both by performing policy gradient updates on the combined advantage:
\begin{equation}
\widehat A_{t,k}^{\mathrm{comb}}
=
\widehat A_t^{\mathrm{out}}
+
\alpha\widehat A_{t,k}^{\mathrm{in}},
\qquad
\alpha\geq0.
\label{eq:clean-mixture}
\end{equation}

The two advantage estimators are produced by different critics, so their
relative scales can differ and change throughout training. We therefore
match their scales once per training batch \(\mathcal B\):
\begin{equation}
c_{\mathcal B}
=
\frac{
\operatorname{sd}_{\mathcal B}
\left(\widehat A_t^{\mathrm{out}}\right)
}{
\operatorname{sd}_{\mathcal B}
\left(\widehat A_{t,k}^{\mathrm{in}}\right)
}.
\label{eq:scale-match}
\end{equation}

The final advantage formulation we use in the policy gradient updates is therefore:
\begin{equation}
\widehat A_{t,k}^{\mathrm{comb}}
=
\widehat A_t^{\mathrm{out}}
+
\alpha c_{\mathcal B}
\widehat A_{t,k}^{\mathrm{in}},
\qquad
\alpha\geq0.
\label{eq:combined-advantage}
\end{equation}
After scale matching, \(\alpha\) controls the relative scale of the inner
contribution.

\section{Experiments}
\label{sec:experiments}
We evaluate DIA on four benchmarks: robomimic~\citep{mandlekar2021what}, FurnitureBench~\citep{10.1177/02783649241304789}, Franka Kitchen~\citep{fu2020d4rl,pmlr-v100-gupta20a}, and D3IL Avoid~\citep{jia2024towards}. These benchmarks cover a range of settings, from single-arm manipulation to long-horizon bimanual tasks, multi-stage kitchen manipulation, and multimodal navigation, with both dense and sparse rewards.

\textbf{Baselines.} For each task, we first pretrain BC policies and then fine-tune the same policy using different methods. We primarily compare DIA with DPPO~\citep{ren2025diffusion} across all benchmarks. On robomimic, we additionally compare against QSM~\citep{3692070.3693742}, IDQL~\citep{hansenestruch2023idql}, DQL~\citep{wang2023diffusion}, DIPO~\citep{yang2023policyrepresentationdiffusionprobability}, AWR~\citep{peng2019advantageweightedregressionsimplescalable}, and RWR~\citep{10.1145/1273496.1273590}, which are commonly used RL fine-tuning methods for single-step policies. We also compare against two RL fine-tuning methods designed for action-chunked policies: Q-chunking~\citep{li2025reinforcement} for flow-matching policy, and ResiP\citep{11127442}, which keeps the base policy fixed and predicts a residual action for each action in the chunk.

\subsection{Results}

\begin{figure}[t]
\centering
\includegraphics[width=\linewidth]{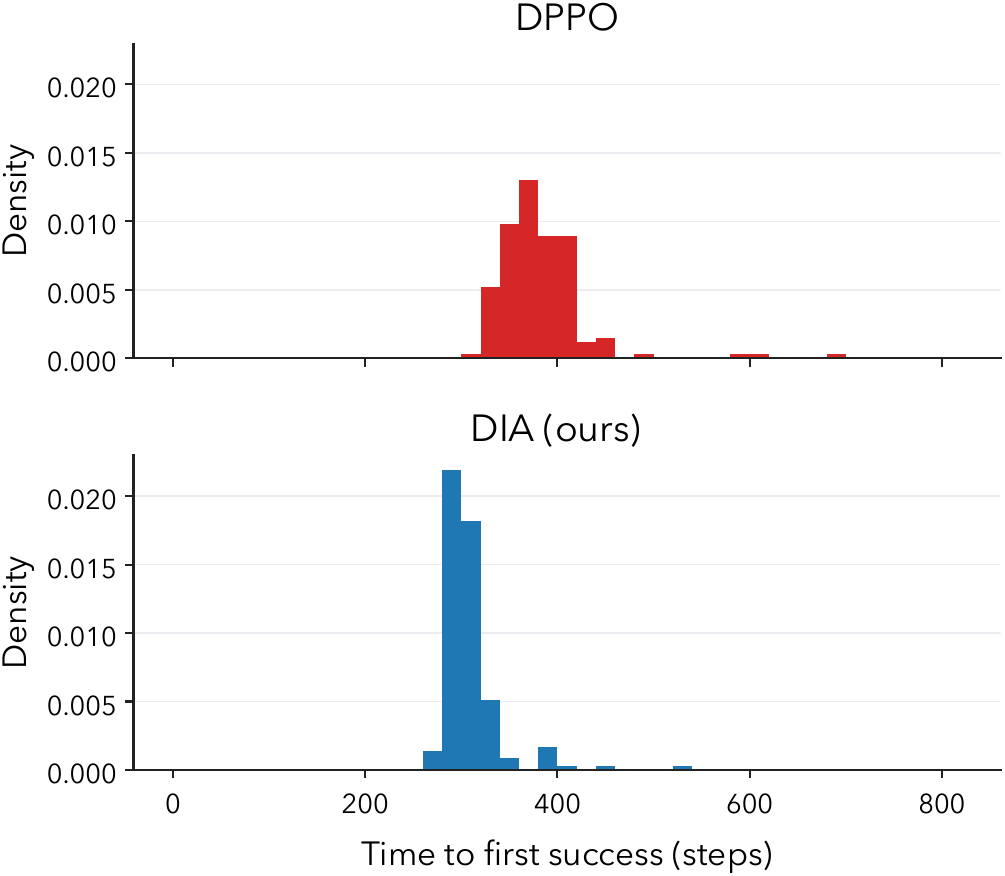}
\caption{Time to first success on \textsc{Transport}, in environment steps, over $200$ evaluation episodes per seed.}
\label{fig:ttc}
\end{figure}
\begin{figure}[t]
\centering
\includegraphics[width=\linewidth]{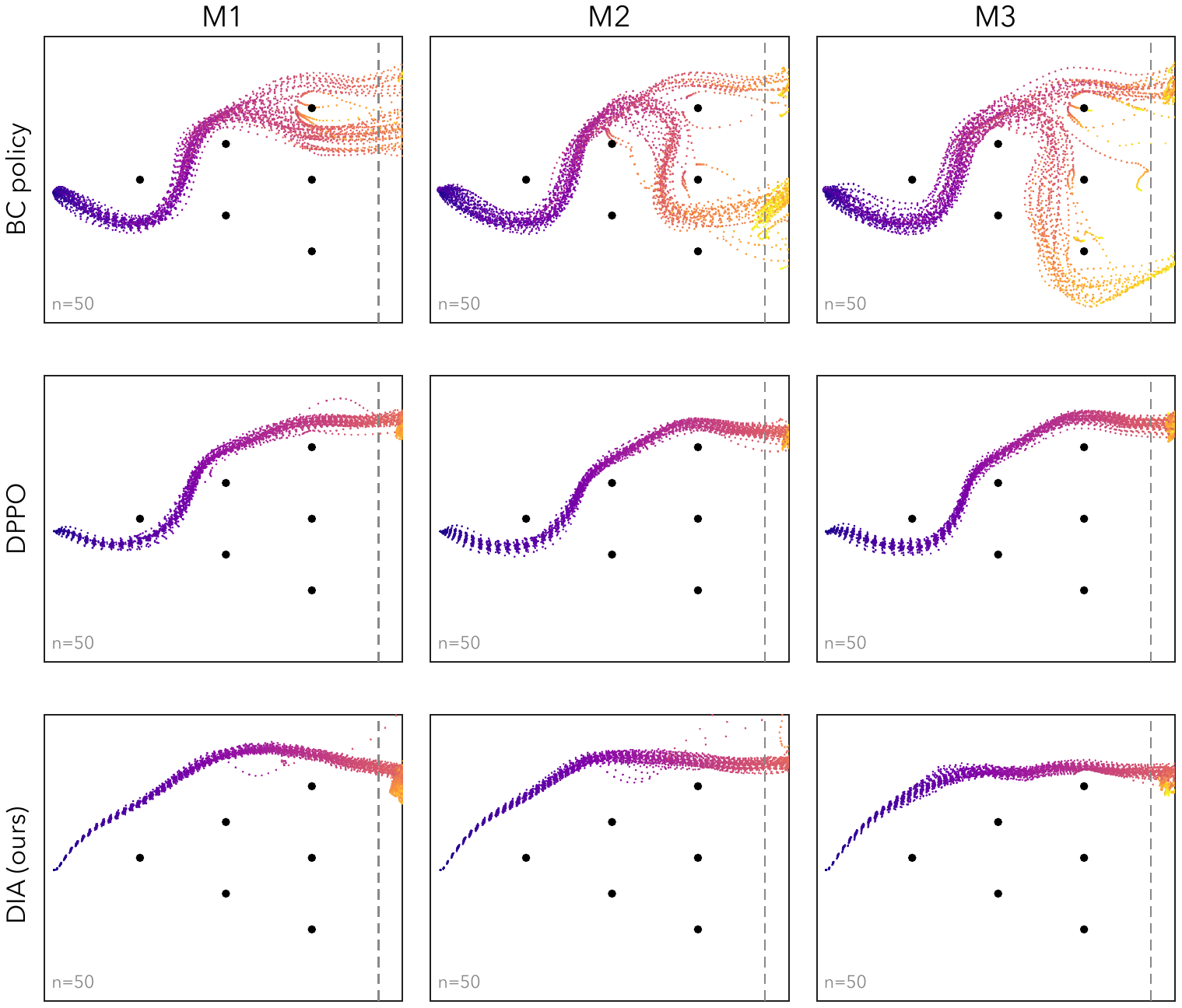}
\caption{D3IL \textsc{avoid}, one row per policy and one column per task variant. Each
panel draws $50$ trajectories as dots colored along the episode. All trajectories are the end
effector.}
\label{fig:avoid_density}
\end{figure}

\textbf{A.1: DIA consistently improves downstream performance across tasks.}
We first evaluate DIA on robomimic, where we run the full set of baseline methods across four manipulation tasks of increasing difficulty: \textsc{Lift}, \textsc{Can}, \textsc{Square}, and bimanual \textsc{Transport} task. DIA achieves the highest final reward among the compared methods, as shown in Table~\ref{tab:robomimic_results} and Figure~\ref{fig:results_ci}.
On robomimic, DIA improves return over DPPO on all four tasks while maintaining similar success rates. The difference is small on the shorter-horizon \textsc{Lift}, \textsc{Can}, and \textsc{Square} tasks. However, for the harder bimanual \textsc{Transport} task, DIA achieves 23\% higher return at nearly the same success rate compared with DPPO. This suggests that DIA improves the quality and efficiency of successful trajectories rather than simply increasing the number of successful episodes. Similar trend also holds for image-based policies, where DIA improves return over DPPO by approximately 8\% on average across \textsc{Square-Pixel} and \textsc{Can-Pixel}.

To further evaluate DIA on sparse longer-horizon tasks, we use FurnitureBench, which contains two long-horizon assembly tasks. Specifically, \textsc{One-Leg} requires attaching a leg to a tabletop, and the more challenging \textsc{Lamp-Med} requires assembling multiple lamp components.
On \textsc{One-Leg}, where performance is already high, DIA improves return by 12\% and success rate by 3\% over DPPO. The gains are larger on the more challenging \textsc{Lamp-Med} task, where DIA improves return by 20\% and success rate by 12\%. This further suggests that DIA provides greater benefit as the task becomes longer-horizon and more difficult.

\textbf{A.2: DIA reaches successful states more efficiently.}

Since all tasks use binary success rewards, higher return at a similar success rate means that DIA reaches successful states earlier. On \textsc{Transport}, where DIA and DPPO have similar success rates, DIA reduces the median time to task completion by $21\%$, as shown in Figure~\ref{fig:ttc}.

To visualize this behavior more clearly, we further examine D3IL \textsc{Avoid}, where the policy trajectory can be directly plotted in the 2D workspace. Across all three pretrained policies, DIA finds shorter and more direct routes to the goal, while DPPO remains closer to the curved trajectories of the pretrained BC policy shown in Figure~\ref{fig:avoid_density}. These results show that DIA improves return not simply by succeeding more often, but by finding more efficient ways to reach successful states.

\begin{figure*}[!t]
\centering
\includegraphics[width=\textwidth]{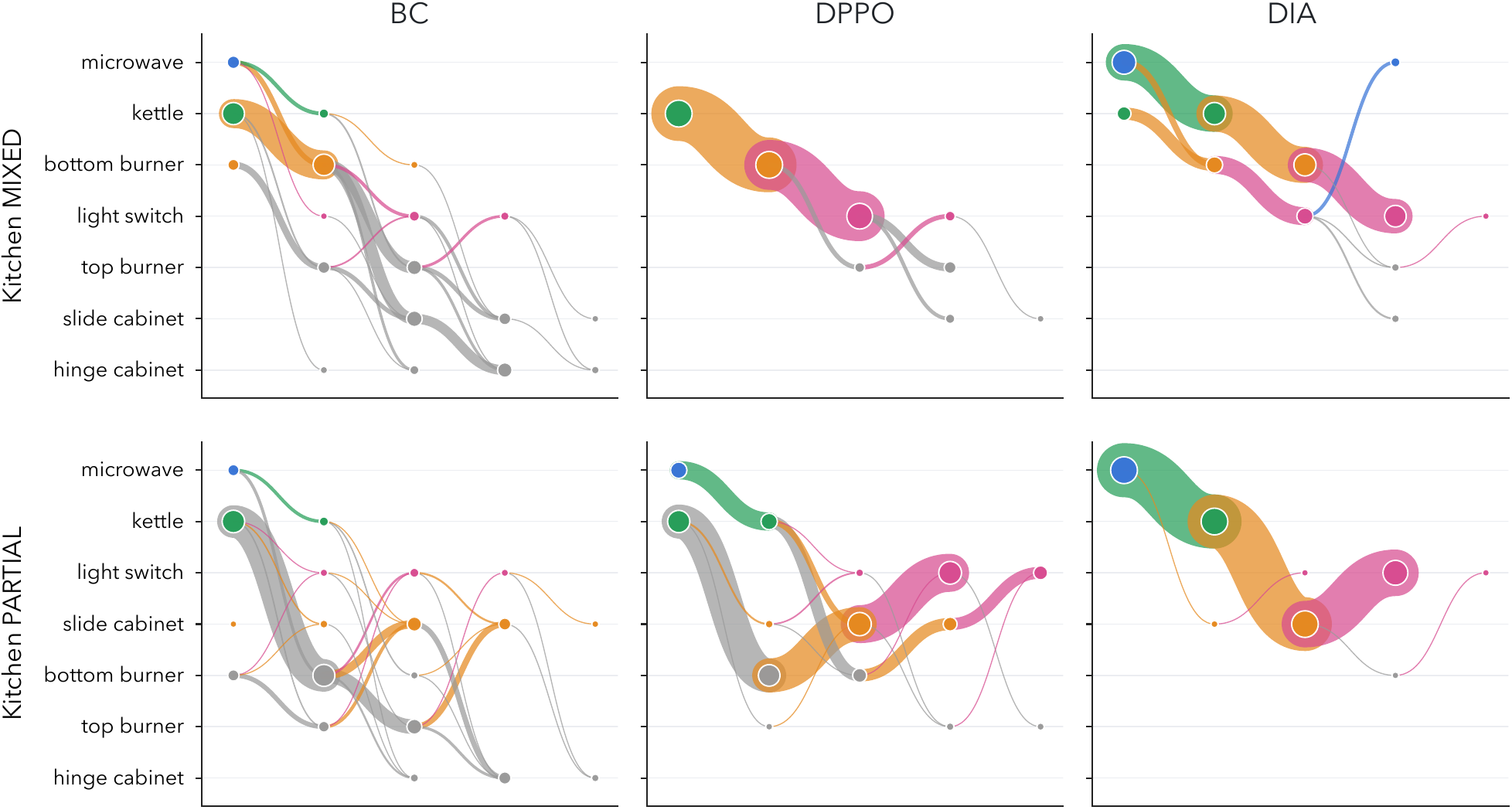}
\caption{Ordered subtask-completion flow on Kitchen \textsc{Mixed} (top) and \textsc{Partial} (bottom); width is the fraction of episodes taking that path.}
\label{fig:kitchen_flow}
\end{figure*}

\textbf{A.3: DIA learns successful task sequences that are weakly represented or absent in the demonstrations.}
We study this behavior on D4RL Kitchen~\citep{fu2020d4rl}, where the complete four-subtask solution becomes progressively less represented from \textsc{Complete} to \textsc{Partial} and \textsc{Mixed}. In \textsc{Complete}, where the full sequence is present in every demonstration, DIA and DPPO perform similarly. In \textsc{Partial}, where only $17$ of $551$ demonstrations contain the full sequence, DIA improves the number of completed subtasks by approximately $15\%$ over DPPO and increases full-task success by 54\%. Figure~\ref{fig:kitchen_flow} shows that DPPO still frequently follows similar sequences to the BC policy which branch into irrelevant subtasks, while DIA shifts toward successful task orderings.

The difference is larger in \textsc{Mixed}, where no demonstration contains the complete target sequence. DIA completes approximately $24\%$ more subtasks than DPPO and achieves $69\%$ full-task success, compared with $0\%$ for both BC and DPPO. This suggests that DIA can move further from the dominant pretrained behavior and recombine separately demonstrated subtasks into successful task sequences.

\section{Conclusion}
\label{sec:conclusion}

We introduce DIA, a
policy-gradient method for fine-tuning diffusion-based robot policies that enables credit assignment within the denoising process. Building on the two-level MDP formulation of diffusion-based policies, DIA defines a conditional value
over partially denoised actions and uses it to construct an inner advantage for
individual denoising decisions. Across Robomimic, FurnitureBench, Franka Kitchen, and D3IL, DIA improves final
performance over baseline methods, with particularly
large gains on challenging long-horizon problems. The experiments also show that
the effect of denoising-level credit extends beyond task return. DIA reaches
successful states through shorter and more direct executions, shift more from dominant behaviors of BC policies. We also demonstrate that when complete task solutions are weakly represented or absent from BC demonstrations, DIA discovers successful task-level sequences. Thus, these results demonstrate the effectiveness of DIA and show the benefit of incorporating intermediate denoising-state values into policy-gradient fine-tuning.

For future work, we aim to extend DIA beyond simulation and evaluate it on physical robots. Since DIA relies on on-policy interaction during fine-tuning, collecting sufficient rollouts can be costly in the real world. Improving sample efficiency and studying sim-to-real transfer are therefore important next steps.

\section*{Acknowledgments}
A.A.-G. thanks Anders~G.~Fr{\o}seth for his generous support. A.A.-G. also acknowledges the generous support of Natural Resources Canada and the Canada 150 Research Chairs program. 
 This research is part of the
University of Toronto’s Acceleration Consortium, which receives funding from the CFREF-2022-00042 Canada
First Research Excellence Fund. This research was enabled in part by compute resources provided by the Vector Institute and the Digital Research Alliance of Canada.


\bibliographystyle{plainnat}
\bibliography{references}

\appendices
\def\thesubsection{\thesection.\arabic{subsection}}
\def\thesubsectiondis{\thesection.\arabic{subsection}}

\section{Policy gradient derivation}
\label{app:policy-grad-deriv}

At environment time $t$ and denoising step $k$, define the score
\begin{equation}
H_{t,k}
:=
\nabla_\theta
\log\pi_{\theta,k}
(x_{t,k+1}\mid s_t,x_{t,k}).
\end{equation}

The environment receives the completed action $a_t$, not the intermediate
denoising states. Once $s_t$ and $a_t$ are fixed, the downstream environment
return is independent of the internal path that produced that action.
Therefore,
\begin{equation}
g_t
\mathrel{\perp\!\!\!\perp}
H_{t,k}
\mid (s_t,a_t).
\end{equation}

Define the completed-action value
\begin{equation}
Q(s_t,a_t)
:=
\Expect{P_\theta}[g_t\mid s_t,a_t],
\end{equation}
where $P_\theta$ denotes the joint on-policy distribution over the environment
trajectory and all denoising variables.

Conditioning on $s_t,a_t$ and using the conditional independence gives
\begin{gather}
\begin{aligned}
\Expect{P_\theta}[H_{t,k}g_t]
&=
\Expect{(s_t,a_t)\sim P_\theta}\!\left[
\Expect{P_\theta}[H_{t,k}g_t\mid s_t,a_t]
\right]\\
&=
\Expect{(s_t,a_t)\sim P_\theta}\!\left[
\Expect{P_\theta}[H_{t,k}\mid s_t,a_t]\,
\Expect{P_\theta}[g_t\mid s_t,a_t]
\right]\\
&=
\Expect{(s_t,a_t)\sim P_\theta}\!\left[
\Expect{P_\theta}[H_{t,k}\mid s_t,a_t]\,
Q(s_t,a_t)
\right]\\
&=
\Expect{P_\theta}[H_{t,k}Q(s_t,a_t)].
\end{aligned}
\\
\boxed{
\Expect{P_\theta}[H_{t,k}g_t]
=
\Expect{P_\theta}[H_{t,k}Q(s_t,a_t)].
}
\end{gather}

By Proposition~\ref{prop:zero-mean-score}, the score has zero conditional
mean under the distribution that sampled the current choice. Consequently,
for any function $b_k(s_t,x_{t,k})$ that does not depend on the sampled
choice $x_{t,k+1}$,
\begin{equation}
\label{eq:score-zero}
\Expect{P_\theta}\!\left[
H_{t,k}b_k(s_t,x_{t,k})
\right]
=0.
\end{equation}

Taking $b_k(s_t,x_{t,k})=V(s_t)$ gives
\begin{equation}
\Expect{P_\theta}\!\left[
H_{t,k}V(s_t)
\right]
=0.
\end{equation}

Therefore subtracting $V(s_t)$ from $Q(s_t,a_t)$ does not change the
expectation:
\begin{equation}
\begin{aligned}
\Expect{P_\theta}[H_{t,k}Q(s_t,a_t)]
&=
\Expect{P_\theta}\!\left[
H_{t,k}\bigl(Q(s_t,a_t)-V(s_t)\bigr)
\right]\\
&=
\Expect{P_\theta}[H_{t,k}A(s_t,a_t)],
\end{aligned}
\end{equation}
where
\begin{equation}
A(s_t,a_t):=Q(s_t,a_t)-V(s_t).
\end{equation}

The expression has now become the ordinary outer-advantage policy-gradient
term:
\begin{equation}
\boxed{
\Expect{P_\theta}[H_{t,k}g_t]
=
\Expect{P_\theta}[H_{t,k}A(s_t,a_t)].
}
\end{equation}

The outer advantage uses $V(s_t)$, which knows nothing about the partially
denoised action. At denoising step $k$, we can instead use the conditional value
\begin{equation}
V_k(s_t,x_{t,k})
:=
\Expect{P_\theta}[g_t\mid s_t,x_{t,k}].
\end{equation}

The difference $V_k(s_t,x_{t,k})-V(s_t)$ does not depend on $x_{t,k+1}$.
Applying equation \eqref{eq:score-zero} with
$b_k(s_t,x_{t,k})=V_k(s_t,x_{t,k})-V(s_t)$ gives
\begin{equation}
\label{eq:vk-minus-v-zero}
\Expect{P_\theta}\!\left[
H_{t,k}\bigl(V_k(s_t,x_{t,k})-V(s_t)\bigr)
\right]
=0.
\end{equation}

The outer advantage and the terminal-$Q$ weight differ by exactly this
zero-contribution term:
\begin{equation}
\label{eq:adv-difference}
A(s_t,a_t)
-
\bigl(Q(s_t,a_t)-V_k(s_t,x_{t,k})\bigr)
=
V_k(s_t,x_{t,k})-V(s_t).
\end{equation}

Equations \eqref{eq:vk-minus-v-zero} and \eqref{eq:adv-difference} therefore give
\begin{equation}
\label{eq:terminal-q}
\boxed{
\Expect{P_\theta}[H_{t,k}A(s_t,a_t)]
=
\Expect{P_\theta}\!\left[
H_{t,k}\bigl(Q(s_t,a_t)-V_k(s_t,x_{t,k})\bigr)
\right].
}
\end{equation}

For every denoising decision $j$, define its exact inner advantage as the
one-step change in chain value:
\begin{equation}
A_j(s_t,x_{t,j},x_{t,j+1})
:=
V_{j+1}(s_t,x_{t,j+1})
-V_j(s_t,x_{t,j}).
\end{equation}

Along one sampled chain, summing all the remaining inner advantages cancels
every intermediate value. Since the terminal chain value is
$V_K(s_t,a_t)=Q(s_t,a_t)$,
\begin{equation}
\label{eq:telescope}
\begin{aligned}
&\sum_{j=k}^{K-1}
A_j(s_t,x_{t,j},x_{t,j+1})\\
&\quad=
\sum_{j=k}^{K-1}
\bigl[
V_{j+1}(s_t,x_{t,j+1})
-V_j(s_t,x_{t,j})
\bigr]\\
&\quad=
V_K(s_t,a_t)-V_k(s_t,x_{t,k})\\
&\quad=
Q(s_t,a_t)-V_k(s_t,x_{t,k}).
\end{aligned}
\end{equation}

Substituting equation \eqref{eq:telescope} into the
right-hand side of equation \eqref{eq:terminal-q} gives
\begin{equation}
\label{eq:split}
\begin{aligned}
&\Expect{P_\theta}\!\left[
H_{t,k}\bigl(Q(s_t,a_t)-V_k(s_t,x_{t,k})\bigr)
\right]\\
&\quad=
\Expect{P_\theta}\!\left[
H_{t,k}
\sum_{j=k}^{K-1}A_j(s_t,x_{t,j},x_{t,j+1})
\right]\\
&\quad=
\Expect{P_\theta}\!\left[H_{t,k}A_k\right]
+
\sum_{j=k+1}^{K-1}
\Expect{P_\theta}\!\left[H_{t,k}A_j\right].
\end{aligned}
\end{equation}

Before decision $j$ is sampled,
its expected inner advantage is zero:
\begin{equation}
\label{eq:inner-bellman}
\begin{aligned}
&\Expect{
x_{t,j+1}\sim
\pi_{\theta,j}(\cdot\mid s_t,x_{t,j})
}\!\left[
A_j(s_t,x_{t,j},x_{t,j+1})
\right]\\
&\quad=
\Expect{
x_{t,j+1}\sim
\pi_{\theta,j}(\cdot\mid s_t,x_{t,j})
}\!\left[
V_{j+1}(s_t,x_{t,j+1})
\right]
-V_j(s_t,x_{t,j})\\
&\quad=
V_j(s_t,x_{t,j})-V_j(s_t,x_{t,j})\\
&\quad=0.
\end{aligned}
\end{equation}

For $j>k$, the earlier score $H_{t,k}$ is already known before decision $j$ is
sampled. Let $\mathcal H_{t,j}=(s_t,x_{t,0},\ldots,x_{t,j})$ denote the chain
history known at that point. Decomposing $P_\theta$ over that history and the
next decision, then applying equation \eqref{eq:inner-bellman}, gives
\begin{equation}
\begin{aligned}
\Expect{P_\theta}\!\left[H_{t,k}A_j\right]
&=
\Expect{\mathcal H_{t,j}\sim P_\theta}\!\left[
H_{t,k}
\Expect{
x_{t,j+1}\sim
\pi_{\theta,j}(\cdot\mid s_t,x_{t,j})
}[A_j]
\right]\\
&=0,
\qquad j>k.
\end{aligned}
\end{equation}

Therefore all the later terms in
equation \eqref{eq:split} vanish, while the current term remains:
\begin{equation}
\boxed{
\begin{aligned}
&
\Expect{P_\theta}\!\left[
H_{t,k}\bigl(Q(s_t,a_t)-V_k(s_t,x_{t,k})\bigr)
\right]
\\
&\quad=
\Expect{P_\theta}\!\left[
H_{t,k}A_k(s_t,x_{t,k},x_{t,k+1})
\right].
\end{aligned}
}
\end{equation}

Combining everything, we get:
\begin{equation}
\boxed{
\begin{aligned}
\Expect{P_\theta}[H_{t,k}g_t]
&=
\Expect{P_\theta}[H_{t,k}Q(s_t,a_t)]\\
&=
\Expect{P_\theta}[H_{t,k}A(s_t,a_t)]\\
&=
\Expect{P_\theta}\!\left[
H_{t,k}\bigl(Q(s_t,a_t)-V_k(s_t,x_{t,k})\bigr)
\right]\\
&=
\Expect{P_\theta}\!\left[
H_{t,k}A_k(s_t,x_{t,k},x_{t,k+1})
\right].
\end{aligned}
}
\end{equation}

\begin{proposition}
\label{prop:zero-mean-score}
Let $p_\theta(u\mid z)$ be a normalized probability density on a
parameter-independent support $\mathcal U$, differentiable in $\theta$ and
such that differentiation may be moved through the integral. Then for every
$z$,
\begin{equation}
\label{eq:prop-score-zero}
\Expect{U\sim p_\theta(\cdot\mid z)}\!\left[
\nabla_\theta\log p_\theta(U\mid z)
\right]=0.
\end{equation}
Consequently, for $Z\sim\mu$ with $U\mid Z\sim p_\theta(\cdot\mid Z)$ and any
function $b(Z)$ that does not depend on the sampled choice $U$,
\begin{equation}
\label{eq:prop-score-baseline}
\Expect{Z\sim\mu,\;U\sim p_\theta(\cdot\mid Z)}\!\left[
\nabla_\theta\log p_\theta(U\mid Z)\,b(Z)
\right]=0.
\end{equation}
\end{proposition}

\begin{proof}
Hold $z$ fixed and integrate the score against the density:
\begin{equation}
\begin{aligned}
\Expect{U\sim p_\theta(\cdot\mid z)}\!\left[
\nabla_\theta\log p_\theta(U\mid z)
\right]
&=\int_{\mathcal U}
p_\theta(u\mid z)\nabla_\theta\log p_\theta(u\mid z)\,du\\
&=\int_{\mathcal U}\nabla_\theta p_\theta(u\mid z)\,du\\
&=\nabla_\theta\int_{\mathcal U}p_\theta(u\mid z)\,du\\
&=\nabla_\theta 1\\
&=0,
\end{aligned}
\end{equation}
which proves \eqref{eq:prop-score-zero}. For
\eqref{eq:prop-score-baseline}, condition on $Z$ and apply
\eqref{eq:prop-score-zero}:
\begin{equation}
\begin{aligned}
&\Expect{Z\sim\mu,\;U\sim p_\theta(\cdot\mid Z)}\!\left[
\nabla_\theta\log p_\theta(U\mid Z)\,b(Z)
\right]\\
&\quad=
\Expect{Z\sim\mu}\!\left[
b(Z)
\Expect{U\sim p_\theta(\cdot\mid Z)}\!\left[
\nabla_\theta\log p_\theta(U\mid Z)
\right]
\right]\\
&\quad=0.
\end{aligned}
\end{equation}
\end{proof}

\section{Design decisions}
\label{app:design}

\begin{figure*}[t]
\centering
\includegraphics[width=\textwidth]{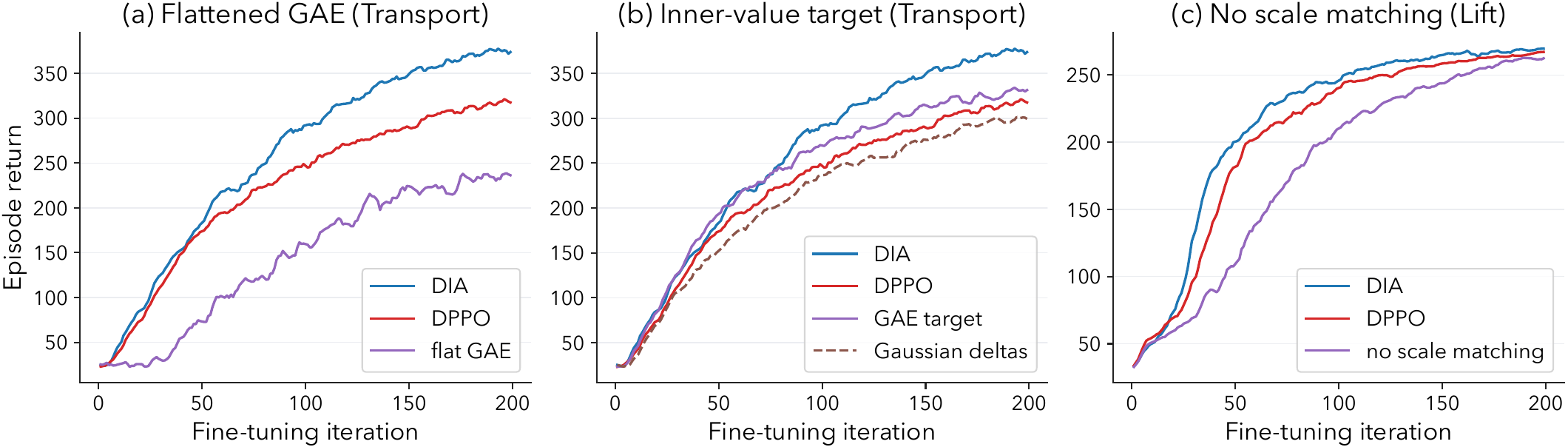}
\caption{Design decision ablations.
(a) One flattened GAE over the combined denoising-and-environment chain, against
DIA's two-level design, on \textsc{Transport}.
(b) Regressing the inner value on the environment GAE return instead of the
action's $Q$, on \textsc{Transport}, together with the Gaussian-delta control.
(c) Combining the two advantages at their raw scale, $\alpha{=}0.5$, on
\textsc{Lift}, single seed. Curves are smoothed over five iterations; (a) and (b)
average five seeds, and each panel legends its own arms.}
\label{fig:ablations}
\end{figure*}

\subsection{Scope of the inner MDP: within-trajectory advantages}
\label{app:innermdp-scope}

Consider instead treating the denoising latents as ordinary states of a single MDP, with
one GAE run across the combined denoising-and-environment chain. Such a value must
forecast, from a noisy intermediate latent $x_{t,k}$, a target that folds in three sources of
stochasticity: the remaining denoising of the current chain, the environment transition,
and the denoising noise of every subsequent action. The first two are well determined by
the action $x_{t,k}$ is forming: the current chain resolves to that action, and the
environment responds accordingly. The third is not. Predicting how the latent $x_{t,k}$
influences the future latents of subsequent denoising chains difficult, since those
states are high-dimensional, highly stochastic, and shaped by many independent factors;
making $x_{t,k}$ attend to this signal yields a high-variance, poorly identified target.
Because the inner GAE takes differences of the value across adjacent denoising
indices, those differences are governed by this cross-step noise rather than by the
effect of the denoising step itself, so the resulting per-step advantages are both noisy
and not causally meaningful. Collapsing the two levels in this way destabilizes training
in practice, even on the simplest robomimic tasks, as we confirm in our ablations.

We therefore adopt a two-level MDP, forming the inner-value advantage
$\widehat A^{\mathrm{in}}_{t,k}$ \emph{within a single denoising trajectory}: for a fixed
environment step $t$, the inner GAE runs over the $K$ denoising indices of the one action
$a_t$, initialized with $\widehat A^{\mathrm{in}}_{t,K}=0$ and with no bootstrapping into the denoising
chain of the next action $a_{t+1}$. Adjacent latent states influence one another only if
they lie on the same denoising chain, so the inner value $\widehat V_k(s_t,x_{t,k})$ is
asked to predict a single well-scoped quantity: the value of the action that \emph{this}
denoising chain will produce.

Collapsing the levels also inflates the effective horizon. A single flat MDP expands
every environment step into its $K$ denoising steps, lengthening a $T$-step trajectory
to one of $TK$ steps. Environment rewards, now separated by a full denoising chain at
every transition, must propagate back across this inflated horizon, where discounting
and the bootstrapping noise of the intervening denoising steps progressively attenuate
them, so the outer reward signal is marginalized. DIA keeps the environment reward MDP
at its true length $T$: the outer advantage $\widehat A^{\mathrm{out}}_t$ is a standard
environment-level GAE over environment steps, so rewards flow back in time undiluted,
while the inner advantage carries the separate within-chain credit at the denoising
level.

\smallskip\noindent\textbf{Ablation:} we run the flat alternative directly, one GAE over the combined
denoising-and-environment chain, with the environment reward on each chain's
boundary step. On \textsc{Transport},
Figure~\ref{fig:ablations}(a), its best setting reaches a train reward $235$ against $375$ for
DIA and $318$ for DPPO. 

\Needspace*{6\baselineskip}
\subsection{Inner-value target: regressing to the action's $Q$}
\label{app:qtarget}

DIA trains the inner value $\widehat V_k(s_t,x_{t,k})$ by regressing it directly to the
$Q$-value of the action the denoising chain produces, $\widehat Q(s_t,a_t)$, from every state
along the chain ($\widehat V_k$ is fit to $\widehat Q$ at every $k$). This
is the natural target following \eqref{eq:Vk-def}, we define the value of an
intermediate denoising state under zero inner reward as the expected value of the action that
continuing the chain will produce, which is the action's $Q$.

A natural alternative retains the generalized-advantage formulation used for the outer
signal and conditions its critic on the intermediate state. This trains an inner value
$\widehat V_k(s_t,x_{t,k})$ on the environment GAE return $R_t = A_t^{\mathrm{out}} + \widehat V(s_t)$,
the same target used for the outer critic $\widehat V$, but conditioned on the chain
state $x_{t,k}$ and the denoising index $k$ as well,
\begin{equation}
\mathcal{L}_{\mathrm{alt}}(\omega) = \mathbb{E}_{t,k}\big[\big(\widehat V_k(s_t,x_{t,k}) - R_t\big)^2\big].
\label{eq:altloss}
\end{equation}
It forms the inner advantage from $\widehat V_k$ through the inner GAE recursion,
\begin{equation}
\begin{aligned}
\widehat\delta_{t,k}
&= \widehat V_{k+1}(s_t,x_{t,k+1}) - \widehat V_k(s_t,x_{t,k}), \\
\widehat A^{\mathrm{in}}_{t,k}
&= \widehat\delta_{t,k}
 + \lambda_{\mathrm{in}}\,\widehat A^{\mathrm{in}}_{t,k+1},
\quad \widehat A^{\mathrm{in}}_{t,K}=0.
\end{aligned}
\label{eq:altgae}
\end{equation}

The two targets have the same population regression function. With an exact outer critic and on-policy data the outer
TD residual is $\delta_t=r_t+\gamma V(s_{t+1})-V(s_t)$ and the GAE return of
\eqref{eq:altloss} is $R_t=\sum_{l\ge0}(\gamma\lambda)^l\delta_{t+l}+V(s_t)$. Its terms
then satisfy
\begin{equation}
\mathbb E[\delta_t\mid s_t,a_t]
= Q(s_t,a_t)-V(s_t)=A(s_t,a_t),
\label{eq:delta-now}
\end{equation}
by the Bellman equation for $Q$. For $l\ge1$ we condition on $s_{t+l}$ alone, which
averages over the action as well as the transition:
\begin{equation}
\begin{aligned}
\mathbb E&[\delta_{t+l}\mid s_{t+l}]\\
&=\mathbb E_{a\sim\pi}\big[\mathbb E[r_{t+l}+\gamma V(s_{t+l+1})\mid s_{t+l},a]\big]-V(s_{t+l})\\
&=\mathbb E_{a\sim\pi}\big[Q(s_{t+l},a)\big]-V(s_{t+l})\\
&=V(s_{t+l})-V(s_{t+l})=0,
\end{aligned}
\label{eq:delta-later}
\end{equation}
the second line by the same Bellman equation and the third because $V(s)$ is by
definition $\mathbb E_{a\sim\pi}[Q(s,a)]$, which is where on-policy data is needed. By the Markov property and the tower rule
\eqref{eq:delta-later} therefore makes every $l\ge1$ term vanish under
$\mathbb E[\,\cdot\mid s_t,a_t]$, so only $l=0$ survives and $\lambda$ drops out:
\begin{equation}
\mathbb E[R_t\mid s_t,a_t]=A(s_t,a_t)+V(s_t)=Q(s_t,a_t).
\label{eq:R-is-Q}
\end{equation}
Averaging \eqref{eq:R-is-Q} over the actions the remainder of the chain can still produce,
\begin{equation}
\begin{aligned}
\mathbb E[R_t\mid s_t,x_{t,k},k]
&=\mathbb E\big[\,\mathbb E[R_t\mid s_t,a_t]\,\big|\,s_t,x_{t,k},k\big]\\
&=\mathbb E\big[Q(s_t,a_t)\mid s_t,x_{t,k},k\big],
\end{aligned}
\label{eq:same-target}
\end{equation}
which is exactly the target the terminal-$Q$ regression uses. Under squared error both
losses therefore identify the same inner value: neither target is intrinsically more
informative about $x_{t,k}$.

What separates them is the noise a finite sample must average over. Conditioning on the
emitted action and applying the law of total variance,
\begin{equation}
\begin{aligned}
\operatorname{Var}(R_t\mid s_t,x_{t,k})
&=\operatorname{Var}\big(\mathbb E[R_t\mid s_t,a_t]\mid s_t,x_{t,k}\big)\\
&\quad+\mathbb E\big[\operatorname{Var}(R_t\mid s_t,a_t)\mid s_t,x_{t,k}\big]\\
&=\operatorname{Var}\big(Q(s_t,a_t)\mid s_t,x_{t,k}\big)\\
&\quad+\mathbb E\big[\operatorname{Var}(R_t\mid s_t,a_t)\mid s_t,x_{t,k}\big],
\end{aligned}
\label{eq:target-var}
\end{equation}
using \eqref{eq:R-is-Q} in the last step. The first term is the variance in which action
the remainder of the chain produces; the second is the variance of what the environment
does once that action is executed. The terminal-$Q$ target carries only the first, its conditional variance
being $\operatorname{Var}(Q(s_t,a_t)\mid s_t,x_{t,k})$ exactly. Replacing $R_t$ by its
conditional expectation given $(s_t,a_t)$ thus preserves the mean, by \eqref{eq:R-is-Q},
while removing $\mathbb E[\operatorname{Var}(R_t\mid s_t,a_t)\mid s_t,x_{t,k}]$ from the
variance.

The derivation above is stated for the exact $V$ and $Q$, which is where the properties
it uses hold. What is implemented substitutes learned stand-ins for both: $R_t$ is formed
with the outer critic $\widehat V\neq V$, and the terminal-$Q$ target is the ensemble
estimate $\widehat Q\neq Q$. Neither idealisation is exact, so in practice this is a
bias--variance tradeoff rather than a strict improvement: the sampled return stays close
to the realized rewards at high variance, while the terminal-$Q$ target removes that
variance at the cost of whatever bias the $Q$-ensemble carries.

\smallskip\noindent\textbf{Ablation:} We run this alternative directly, training the inner
value on the environment GAE return of \eqref{eq:altloss} and forming the inner advantage
through \eqref{eq:altgae}, with the only change being the target for $\widehat V$. On \textsc{Transport},
Figure~\ref{fig:ablations}(b), it reaches a train reward $332$ against $375$ for DIA over five seeds.

\smallskip\noindent\textbf{Gaussian deltas.} As a control that the gain is not
intrinsic to perturbing the update with an extra scale-matched advantage, we replace
the inner TD residuals $\widehat\delta_{t,k}$ of \eqref{eq:inner-delta} with mean-zero
Gaussian draws $\tilde\delta_{t,k}$, then re-run the
inner GAE and the scale matching over them to obtain $\tilde A^{\mathrm{in}}_{t,k}$.
The substituted term enters the update at DIA's magnitude and with the same profile
across $k$, carrying no information about the chain.

By \eqref{eq:inner-bellman} the real
inner advantage is already conditionally mean-zero given the chain history, so at every
$k$ both terms have the same expectation. Because $\tilde\delta_{t,k}$ is drawn
independently of the sampled transition $x_{t,k+1}$, \eqref{eq:prop-score-baseline}
applies and, holding the batch-level scale factor fixed,
\begin{equation}
\Expect{P_\theta}\!\left[H_{t,k}\,\tilde A^{\mathrm{in}}_{t,k}\right]=0 .
\end{equation}
Empirically it reaches a train reward
of $299$ on \textsc{Transport}
against $318$ for DPPO and $375$ for DIA, recovering none of the gain.

\subsection{Scale matching the inner advantage}
\label{app:scalematch}
The inner advantage $\widehat A^{\mathrm{in}}_{t,k}$ and the outer advantage
$\widehat A^{\mathrm{out}}_t$ are produced by different value functions and need not
share a scale; moreover,
the inner-value scale drifts over training as the $Q$-ensemble grows. To keep the
mixing hyperparameter $\alpha$ in \eqref{eq:clean-mixture} interpretable, we rescale the inner
advantage to match the per-iteration standard deviation of the outer advantage
before combining:
\begin{equation}
c_{\mathcal B}\,\widehat A_{t,k}^{\mathrm{in}},
\qquad
c_{\mathcal B}
=
\frac{\operatorname{sd}_{\mathcal B}\big(\widehat A_t^{\mathrm{out}}\big)}
     {\operatorname{sd}_{\mathcal B}\big(\widehat A_{t,k}^{\mathrm{in}}\big)}.
\end{equation}
With this rescaling the standard deviation of the inner contribution
$\alpha c_{\mathcal B}\widehat A^{\mathrm{in}}_{t,k}$ is exactly $\alpha$ times that of the outer
advantage, so $\alpha$ has a fixed meaning (the fraction of inner contribution
relative to the outer signal), independent of the $Q$ scale at any given
iteration. Both standard deviations are computed globally over the batch.

\smallskip\noindent\textbf{Ablation:}\label{app:scalematch-ablation} We ablate the rescaling by adding the two advantages at their raw scale, at
$\alpha=0.5$. The mixture is then set by whichever term is larger in raw units, and
since the inner scale follows the $Q$-ensemble it moves over training, so $\alpha$ no
longer fixes the inner-to-outer ratio at any point in the run. On \textsc{Lift} the two
are more than an order of magnitude apart, the inner advantage averaging $3.6\%$ of the
outer, so at $\alpha{=}0.5$ it supplies between $0.7\%$ and $3.5\%$ of the mixed
advantage over the whole run. The ablation ends at $262.1$, next to $266.4$ for DPPO and $271.0$
for DIA (Figure~\ref{fig:ablations}(c), single seed).

\begin{figure*}[t]
\centering
\includegraphics[width=\textwidth]{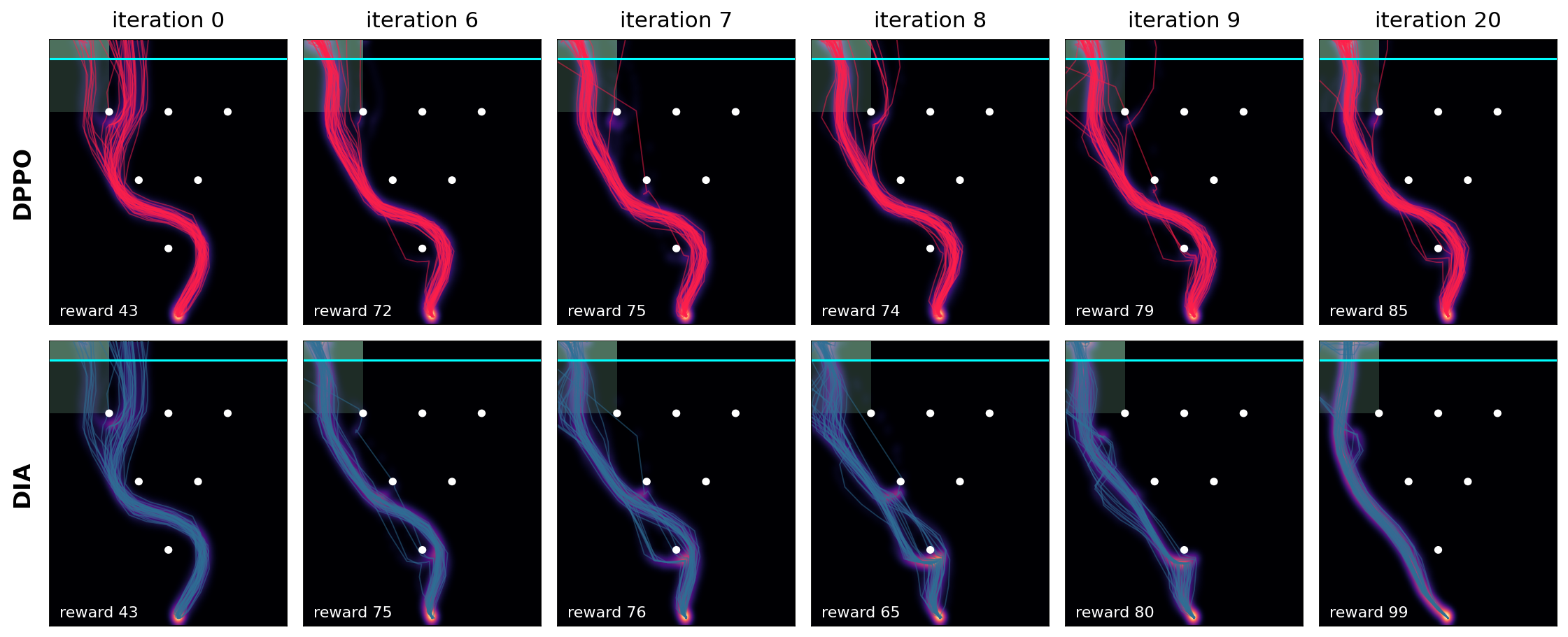}
\caption{D3IL \textsc{avoid} M1, seed $42$, at six fine-tuning iterations, with episode
reward per panel. Trajectory density with $35$ sampled paths drawn over it. DIA (bottom)
migrates to the far side of the first obstacle across iterations $7$ to $9$; DPPO (top),
on the same states, keeps the pretrained route throughout. Appendix~\ref{app:commitment}
probes the critics at iteration $7$, the crossover.}
\label{fig:avoid_migration}
\end{figure*}

\section{Probing the inner value over the denoising chain}
\label{app:perk-diagnostic}
A natural question is whether $\widehat V_k$ actually reads the chain state $x_{t,k}$ and
produces informative per-$k$ predictions, or merely outputs a state-conditioned
constant with noise. We probe this on a saved $\widehat V_k$ checkpoint
(\textsc{Square}, seed $43$, iteration $90$) by rolling out one successful episode
($T=100$, first reward at $t=26$) and evaluating $\widehat V_k(s_t,x_{t,k})$ for the
policy-sampled denoising chain at $12$ timesteps spanning pre-reward and post-reward
states.

\subsection{Within-trajectory $\widehat V_k$ predictions}
For each picked $t$, we sample one policy chain and evaluate $\widehat V_k$ at every
chain position.
\begin{center}
\small
\begin{tabular}{rrrr}
\toprule
$t$ & $\widehat V_0$ & $\widehat V_K$ & $\Delta_{k=0\to K}$ \\
\midrule
0   & $-23.32$  & $-16.26$  & $+7.06$  \\
3   & $+32.15$  & $+30.30$  & $-1.85$  \\
6   & $+79.56$  & $+93.99$  & $+14.42$ \\
9   & $+92.60$  & $+99.87$  & $+7.27$  \\
12  & $+77.19$  & $+97.19$  & $+20.00$ \\
\textbf{15}  & $\mathbf{+126.59}$ & $\mathbf{+157.25}$ & $\mathbf{+30.66}$ \\
18  & $+190.60$ & $+196.36$ & $+5.75$  \\
21  & $+213.75$ & $+222.11$ & $+8.36$  \\
24  & $+250.27$ & $+249.97$ & $-0.29$  \\
\midrule
\multicolumn{4}{l}{\emph{Post first reward ($t = 26$): action choice no longer matters}} \\
\midrule
27  & $+263.80$ & $+264.54$ & $+0.74$  \\
33  & $+266.84$ & $+268.02$ & $+1.18$  \\
\bottomrule
\end{tabular}
\end{center}

\noindent
$\widehat V_k$ moves along the chain at fixed $t$, with $\Delta_{k=0\to K}$ averaging
$11.9$ in magnitude over $t\le21$ and peaking at $+30.66$, so the critic reads $x_{t,k}$
rather than emitting a state-conditioned constant with noise. That movement also tracks
whether the action still matters: from $t=24$ on, with the first reward collected at
$t=26$, $\Delta_{k=0\to K}$ is at most $1.18$ in magnitude, an order of magnitude below
its pre-reward values, correctly indicating action choice no longer influences reward.

\subsection{$\widehat V_k$ reacts as the chain denoises}
\label{app:commitment}

On \textsc{avoid} the policy changes which side of the first obstacle it passes during
the first ten updates, which gives an explicit behavioural choice to attach the inner
value to. Rolling out $50$ episodes per checkpoint of a single run (M1, seed $42$) and
recording which side of obstacle~1 the end effector passes:

\begin{center}
\small
\begin{tabular}{l r r r r r r}
\toprule
iteration & $0$ & $6$ & $7$ & $8$ & $9$ & $20$ \\
\midrule
episodes passing left & $0$ & $6$ & $24$ & $31$ & $50$ & $50$ \\
\bottomrule
\end{tabular}
\end{center}

\noindent
The pretrained policy always goes right; by iteration $9$ the policy always goes left, and
iteration $7$ is the crossover. Figure~\ref{fig:avoid_migration} shows the paths
either side of it, with DPPO on the same states for contrast.

We take the iteration-$7$ checkpoint, fix one environment state $s$, and draw $N{=}400$
denoising chains from a \emph{single shared latent} at chain position $k{=}0$, so every
sample begins identically and all divergence is produced by the fine-tuned denoising
steps. Splitting the samples at the median lateral commitment of the emitted action
chunk, and reading the critics on each:

\begin{center}
\small\setlength{\tabcolsep}{4pt}
\begin{tabular}{l r r r}
\toprule
 & left group & right group & gap \\
\midrule
$\widehat Q(s,a_0)$ & $+0.198$ & $-0.166$ & $-0.364$ \\
\midrule
$\widehat V_{0}$  & $+0.690$ & $+0.690$ & $\phantom{-}0.000$ \\
$\widehat V_{2}$                     & $+1.048$ & $+0.727$ & $-0.320$ \\
$\widehat V_{4}$                     & $+0.861$ & $+0.299$ & $-0.563$ \\
$\widehat V_{6}$                     & $+0.713$ & $-0.070$ & $-0.782$ \\
$\widehat V_{8}$                     & $+0.455$ & $-0.335$ & $-0.789$ \\
$\widehat V_{10}$                    & $+0.716$ & $-0.278$ & $-0.994$ \\
\bottomrule
\end{tabular}
\end{center}

\noindent
$\widehat V(s)$ is $10.00$ for every sample by construction, and the value at $k{=}0$ is identical across samples because we fixed the initial latent. The separation then grows monotonically along the chain, which tracks the difference in route as the chain resolves toward it. 

We can see the effect as fine tuning proceeds, the policy steered by $\widehat V_k$ commits fully to the left side route in subsequent training iterations.

\section{Wall-clock overhead}
\label{app:wallclock}
Table~\ref{tab:wallclock} reports the per-iteration wall time of every method on the
robomimic tasks. DIA is the most expensive, $2$--$5\%$ above DPPO. Every entry is measured sequentially on the same RTX 3090 at $n_{\mathrm{env}}{=}50$.

\begin{table}[h]
\centering
\caption{Wall-clock time per iteration (seconds) of DIA versus the other methods.}
\label{tab:wallclock}
\small
\setlength{\tabcolsep}{5pt}
\begin{tabular}{lcccc}
\toprule
Method & \textsc{Lift} & \textsc{Can} & \textsc{Square} & \textsc{Transport} \\
\midrule
DIA (ours) & $90$ & $97$ & $141$ & $369$ \\
DPPO & $87$ & $93$ & $134$ & $362$ \\
IDQL & $47$ & $54$ & $80$ & $349$ \\
DQL & $47$ & $51$ & $80$ & $347$ \\
QSM & $40$ & $46$ & $69$ & $342$ \\
RWR & $37$ & $43$ & $66$ & $331$ \\
DIPO & $44$ & $51$ & $76$ & $341$ \\
AWR & $38$ & $45$ & $68$ & $333$ \\
\bottomrule
\end{tabular}
\end{table}

\section{Additional task and training details}
\label{app:tasks}

We evaluate on four benchmarks which span low and high-dimensional observations and horizons from $100$ to $1000$ environment steps. Table~\ref{tab:tasks} lists every task and
Table~\ref{tab:trainsettings} the optimization settings. Within a task, every
method we compare shares the actor architecture, the
pre-trained behavior-cloning checkpoint, the observation normalization and
every environment setting, so the comparison isolates the fine-tuning
algorithm. The one exception is QC~\citep{li2025reinforcement}, which we run in its own framework with its own flow policy and therefore does not share the pre-trained checkpoint of the other diffusion based methods.

\begin{table*}[t]
\centering
\caption{Tasks used in our experiments. $|\mathcal{O}|$ and $|\mathcal{A}|$ are the per-step observation and action dimensions, $H$ the environment-step horizon.}
\label{tab:tasks}
\small
\begin{tabular}{llccr}
\toprule
Benchmark & Task & $|\mathcal{O}|$ & $|\mathcal{A}|$ & $H$ \\
\midrule
\multirow{4}{*}{Robomimic}
  & \textsc{Lift}      & $19$ & $7$  & $300$  \\
  & \textsc{Can}       & $23$ & $7$  & $300$  \\
  & \textsc{Square}    & $23$ & $7$  & $400$  \\
  & \textsc{Transport} & $59$ & $14$ & $800$  \\
\midrule
\multirow{2}{*}{FurnitureBench}
  & \textsc{One-leg} (med) & $58$ & $10$ & $700$  \\
  & \textsc{Lamp} (med)    & $44$ & $10$ & $1000$ \\
\midrule
\multirow{3}{*}{D4RL Kitchen}
  & \textsc{Complete} & $60$ & $9$ & $280$ \\
  & \textsc{Mixed}    & $60$ & $9$ & $280$ \\
  & \textsc{Partial}  & $60$ & $9$ & $280$ \\
\midrule
\multirow{3}{*}{D3IL \textsc{Avoid}}
  & M1 & $4$ & $2$ & $100$ \\
  & M2 & $4$ & $2$ & $100$ \\
  & M3 & $4$ & $2$ & $100$ \\
\bottomrule
\end{tabular}
\end{table*}

\begin{table*}[t]
\centering
\caption{Training settings, shared across the tasks of a benchmark unless noted.
$T_a$ is the action chunk and $T_o$ the observation history, $K$ the number of
denoising steps of the pre-trained policy and $K_{\mathrm{ft}}$ the number that
are fine-tuned. $N_{\mathrm{env}}$ is the number of parallel environments,
$N_{\mathrm{itr}}$ the number of fine-tuning iterations and $N_{\mathrm{warm}}$
the number of critic-only warmup iterations before the actor is updated. The
learning rates and $\alpha$ are identical across all four benchmarks.}
\label{tab:trainsettings}
\small
\begin{tabular}{lcccccccccc}
\toprule
Benchmark & $T_a$ & $T_o$ & $K$ & $K_{\mathrm{ft}}$ & $\gamma$ & $N_{\mathrm{env}}$ & $N_{\mathrm{itr}}$ & $N_{\mathrm{warm}}$ & actor lr & critic lr \\
\midrule
Robomimic           & $4^{\dagger}$ & $1$ & $20$  & $10$ & $0.999$ & $50$   & $200^{\ddagger}$ & $2^{\ddagger}$ & $10^{-4}$ & $10^{-3}$ \\
FurnitureBench      & $8$           & $1$ & $100$ & $5$  & $0.999$ & $1000$ & $1000$           & $1$            & $10^{-5}$ & $10^{-3}$ \\
D4RL Kitchen        & $4$           & $1$ & $20$  & $10$ & $0.99$  & $40$   & $600$            & $0$            & $10^{-4}$ & $10^{-3}$ \\
D3IL \textsc{Avoid} & $4$           & $1$ & $20$  & $10$ & $0.99$  & $50$   & $50$             & $1^{\S}$       & $10^{-5}$ & $10^{-3}$ \\
\bottomrule
\end{tabular}
\\[2pt]
\footnotesize $^{\dagger}$\textsc{Transport} uses $T_a{=}8$. \quad
$^{\ddagger}$\textsc{Transport} is read at iteration $180$ and uses
$N_{\mathrm{warm}}{=}5$. \quad
$^{\S}$M3 uses $N_{\mathrm{warm}}{=}2$. \quad
All runs use $\alpha{=}0.5$ and a $Q$ learning rate of $3\times10^{-4}$ except
\textsc{Kitchen Partial}, which uses $10^{-3}$ (Appendix~\ref{app:tasks-kitchen}).
\end{table*}

\subsection{Robomimic}
\label{app:tasks-robomimic}

We evaluate on four manipulation tasks from the Robomimic~\citep{mandlekar2021what} benchmark, with the multi-human demonstration split and
the state-based observations.  \textsc{Lift} and \textsc{Can} are single-arm pick-and-place, \textsc{Square} is a peg-in-hole insertion with a tight tolerance, and \textsc{Transport} is a bi-manual multi-stage manipulation task. The reward is a per-step binary success indicator. 

\smallskip\noindent\textbf{Randomization.} Object placement is resampled at every
reset by robosuite's uniform randomizer. \textsc{Lift} samples the cube within $\pm3$\,cm in
$x$ and $y$ about the table centre with unconstrained rotation about $z$. \textsc{Can} samples
the object over $\pm14.5$\,cm in $x$ and $\pm19.5$\,cm in $y$ about the source bin's centre,
these being the arena's table half-extents less $5$\,cm, also with unconstrained rotation
about $z$. \textsc{Square} samples the nut over $x\in[-0.115,-0.11]$\,m and
$y\in[0.11,0.225]$\,m with unconstrained rotation about $z$. \textsc{Transport} places all six objects within
$\pm5$\,mm in $x$ and $y$ and rotates two of them: the payload by
$\pm\pi/6$ about $y$, taken about a $\pi/2$ centre, and the trash by $\pm0.3\pi$ about $z$. The two bins,
the lid and the trash bin are not rotated. The arm's initial joint
configuration carries Gaussian noise of magnitude $0.02$.

\smallskip\noindent\textbf{Demonstration data.} Every task is pre-trained on the multi-human (\texttt{mh})
split, $300$ teleoperated demonstrations per task collected by six operators of
differing skill, the data is multi-modal both in strategy and in
proficiency. This is the harder of the two splits Robomimic provides: the
proficient-human split is cleaner but far less diverse, and the multi-human
demonstrations contain the hesitations, retries and suboptimal grasps. On \textsc{Square} the demonstrations
comprise $80{,}731$ transitions over episodes of $123$ to $1051$ steps (median
$239$); on \textsc{Transport}, $195{,}800$ transitions over episodes of $392$ to
$2614$ steps (median $630$).

The observation is a flat proprioceptive-plus-object vector. Each
arm contributes end-effector position ($3$), end-effector orientation as a
quaternion ($4$) and gripper joint positions ($2$), and the remainder is the
task's object state: $10$ dimensions on \textsc{Lift} (the cube's pose and its
pose relative to the gripper), $14$ on \textsc{Can} and \textsc{Square}, and
$41$ on \textsc{Transport}, which also carries two arms worth of
proprioception. Actions are $7$-dimensional operational-space deltas per arm
(position, orientation and gripper), doubled to $14$ on \textsc{Transport}.
Observations and actions are scaled to $[-1,1]$ per dimension by the minimum and
maximum of the demonstration data, applied both to the pre-training data and to
every environment step during fine-tuning.

\subsection{FurnitureBench}
\label{app:tasks-furniture}

\textsc{One-leg} and \textsc{Lamp} are long-horizon assembly tasks from
FurnitureBench~\citep{10.1177/02783649241304789} at medium randomization, simulated in
IsaacGym, with horizons of $700$ and $1000$ environment steps.

\smallskip\noindent\textbf{Randomization.} We use the medium setting. At reset the
parts are placed at fixed poses and then perturbed by random impulses, of force up
to $0.5$ and torque up to $0.01$, drawn independently per part; the obstacle is
displaced by up to $\pm4$\,cm in the table plane, and each of the arm's seven
joints is offset by up to $\pm10^\circ$. 

\smallskip\noindent\textbf{Demonstration data.} The demonstrations are teleoperated assembly
episodes recorded at medium initial-state randomization, which perturbs the
starting poses of the parts on the table. \textsc{One-leg} is pre-trained on
$50$ demonstrations totalling $24{,}091$ transitions, with episodes of $362$ to
$624$ steps (median $478$).

The observation concatenates the robot's proprioceptive state, with its
end-effector orientation converted from a quaternion to a $6$D rotation
representation, with the poses of every furniture part in the scene.
\textsc{One-leg} observes $58$ dimensions and \textsc{Lamp} $44$. Actions are
$10$-dimensional: a $3$D position delta, a $6$D rotation delta and a gripper
command. Observations and actions are scaled to $[-1,1]$ per dimension by the
minimum and maximum of the demonstration data, applied both to the pre-training
data and to every environment step during fine-tuning. 

\subsection{Franka Kitchen}
\label{app:tasks-kitchen}

The three D4RL Kitchen datasets~\citep{fu2020d4rl} share an environment and differ
in their demonstration data and in which subtasks are targeted.
\textsc{Complete} and \textsc{Partial} target the microwave, kettle, light
switch and slide cabinet. \textsc{Mixed} targets the bottom burner in place of the slide
cabinet. The return is the number of subtasks solved, with a maximum of four.

\smallskip\noindent\textbf{Randomization.} The simulator state is the same at every
reset: the arm and every object return to one fixed configuration and the goal is fixed,
so the physical scene never varies. Episodes still differ, because the environment adds
uniform sensor noise to every observation it returns, the first one included. The noise
is $\pm0.01$ on each arm joint position and velocity and on each object dimension, the
product of the environment's noise ratio $0.1$ and the per-joint amplitude $0.1$. Across
resets this moves the observed state by up to $0.019$, and leaves the $30$ goal
dimensions untouched. It is the only source of episode-to-episode variation on this
benchmark.

\smallskip\noindent\textbf{Demonstration data.} The three datasets are the D4RL relay-policy
demonstrations. \textsc{Complete} contains $17$ demonstrations, $3290$
transitions, each a clean episode of $182$ to $207$ steps solving all four
target subtasks in order. \textsc{Partial} and \textsc{Mixed} draw on the same
undirected play dataset of roughly $550$ demonstrations and $123{,}000$
transitions, with episodes of $8$ to $280$ steps. In \textsc{Partial}, $17$
demonstrations complete all four target subtasks, $215$ complete three of the
four and $237$ complete two. In \textsc{Mixed} no demonstration completes all
four, and $233$ complete three. Every demonstration in both datasets completes
at least one target subtask.

The observation is $60$-dimensional: $30$ dimensions of state, being $9$ arm
joint positions and $21$ object joint positions, followed by $30$ encoding the
goal configuration. The seven manipulable elements occupy fixed slices of the
object block, the kettle at indices $23$--$29$ and the microwave at index $22$,
and a subtask counts as solved when its slice comes within a fixed threshold of
the goal. Actions are $9$-dimensional joint velocity commands.
Observations and actions are scaled to $[-1,1]$ per dimension, applied both to
the pre-training data and to every environment step during fine-tuning. The
scaling constants are the per-dimension minimum and maximum of the demonstration
data, so the pre-training data lies in $[-1,1]$ by construction. During
environment interaction an observation can fall outside the demonstration range,
and the scaled value is clipped to $[-10,10]$.

\subsection{D3IL Avoid}
\label{app:tasks-avoid}

\textsc{Avoid}~\citep{jia2024towards} is a planar reaching task in which the end
effector must cross a field of six obstacles to reach a goal line. The three
variants M1, M2 and M3 differ only in their demonstration data. The reward is $+1$
per step once the end effector is past the goal line, with a permanent $-0.1$
on contact with an obstacle or on leaving the arena.

\smallskip\noindent\textbf{Randomization.} None. Reset returns the arm to a
fixed joint configuration and the six obstacles are at fixed positions.

\smallskip\noindent\textbf{Demonstration data.} The demonstration set is $96$ human-teleoperated
trajectories, $76$ for training and $20$ held out, recorded on the physical
setup the simulator replicates, each at most $200$ steps. For every obstacle a
human sometimes passes left and sometimes right, so the set covers many distinct
routes to the goal. M1, M2 and M3 are pre-trained on progressively larger
subsets of these routes.

The observation is four-dimensional, the previous two-dimensional action
followed by the end-effector's $(x,y)$ position, and the action is a
two-dimensional end-effector command. Observations and actions are scaled to
$[-1,1]$ per dimension by the minimum and maximum of the demonstration data,
applied both to the pre-training data and to every environment step during
fine-tuning.

\subsection{Listed training hyperparameters}
\label{app:hyperparams}

\autoref{tab:hyperparams} lists every training hyperparameter of the DIA
runs: the settings shared across all tasks, then the per-task values for the
rollout, the diffusion policy, PPO, the DIA-specific critics and the network
sizes. DPPO uses the same values wherever the two methods share a
hyperparameter.

\begin{table*}[t]
\centering
\caption{Training hyperparameters of the DIA runs. $[n]^{d}$ denotes $d$ hidden layers of width $n$.}
\label{tab:hyperparams}
\footnotesize\renewcommand{\arraystretch}{0.95}
\begin{tabular}{@{}lc@{\hspace{2.2em}}lc@{\hspace{2.2em}}lc@{}}
\toprule
\multicolumn{6}{@{}l}{\textit{Shared across all tasks}} \\
\midrule
Observation history $T_o$ & $1$
  & Value loss coefficient & $0.5$
  & Gaussian sample clip & $3$ \\
GAE $\lambda$ & $0.95$
  & Target KL & $1$
  & Min.\ log-prob.\ denoising std & $0.1$ \\
Inner GAE $\lambda_{\mathrm{in}}$ & $0.95$
  & $\widehat Q$ target update rate & $1.0$
  & $\widehat V_k$ time embedding & $32$ \\
Inner weight $\alpha$ & $0.5$
  & $\widehat Q$ ensemble statistic & mean
  & Running reward scaling & yes \\
Critic learning rate & $10^{-3}$
  & $\widehat V_k$ learning rate & $10^{-3}$
  & & \\
\bottomrule
\end{tabular}

\vspace{1.0em}

\setlength{\tabcolsep}{4pt}
\resizebox{\textwidth}{!}{%
\begin{tabular}{lcccccccc}
\toprule
& \textsc{Lift} & \textsc{Can} & \textsc{Square} & \textsc{Transport}
& \textsc{One-Leg} & \textsc{Lamp} & Kitchen & \textsc{Avoid} \\
\midrule
\multicolumn{9}{l}{\textit{Rollout}} \\
Parallel environments $N_{\mathrm{env}}$ & $50$ & $50$ & $50$ & $50$ & $1000$ & $1000$ & $40$ & $50$ \\
Episode horizon $H$                      & $300$ & $300$ & $400$ & $800$ & $700$ & $1000$ & $280$ & $100$ \\
Policy decisions per environment per iteration$^{\ast}$ & $300$ & $300$ & $400$ & $400$ & $88$ & $125$ & $70$ & $25$ \\
Fine-tuning iterations $N_{\mathrm{itr}}$ & $200$ & $200$ & $201$ & $201$ & $1000$ & $1000$ & $601$ & $51$ \\
Critic warmup iterations $N_{\mathrm{warm}}$ & $2$ & $2$ & $2$ & $5$ & $1$ & $1$ & $0$ & $1$ \\
\midrule
\multicolumn{9}{l}{\textit{Diffusion policy}} \\
Denoising steps $K$                     & $20$ & $20$ & $20$ & $20$ & $100$ & $100$ & $20$ & $20$ \\
Fine-tuned denoising steps $K_{\mathrm{ft}}$ & $10$ & $10$ & $10$ & $10$ & $5$ & $5$ & $10$ & $10$ \\
Action chunk $T_a$                      & $4$ & $4$ & $4$ & $8$ & $8$ & $8$ & $4$ & $4$ \\
Minimum sampling denoising std          & $0.1$ & $0.1$ & $0.1$ & $0.1$ & $0.04$ & $0.04$ & $0.1$ & $0.1$ \\
\midrule
\multicolumn{9}{l}{\textit{PPO}} \\
Minibatch size                          & $7500$ & $7500$ & $10000$ & $10000$ & $8800$ & $8800$ & $5600$ & $6250$ \\
Update epochs                           & $10$ & $10$ & $10$ & $5$ & $5$ & $5$ & $10$ & $10$ \\
Discount $\gamma$                       & $0.999$ & $0.999$ & $0.999$ & $0.999$ & $0.999$ & $0.999$ & $0.99$ & $0.99$ \\
Denoising discount $\gamma_{\mathrm{den}}$ & $0.99$ & $0.99$ & $0.99$ & $0.99$ & $0.9$ & $0.9$ & $0.99$ & $0.95$ \\
Clipping $\epsilon$, base               & $0.001$ & $0.001$ & $0.001$ & $0.001$ & $0.001$ & $0.001$ & $0.01$ & $0.1$ \\
Clipping $\epsilon$, maximum            & $0.01$ & $0.01$ & $0.01$ & $0.01$ & $0.001$ & $0.001$ & $0.01$ & $0.1$ \\
Clipping $\epsilon$, rate               & $3$ & $3$ & $3$ & $3$ & $3$ & $3$ & $3$ & $1$ \\
Actor learning rate                     & $10^{-4}$ & $10^{-4}$ & $10^{-4}$ & $10^{-4}$ & $10^{-5}$ & $10^{-5}$ & $10^{-4}$ & $10^{-5}$ \\
\midrule
\multicolumn{9}{l}{\textit{DIA}} \\
$\widehat Q$ learning rate              & $3{\times}10^{-4}$ & $3{\times}10^{-4}$ & $3{\times}10^{-4}$ & $3{\times}10^{-4}$ & $3{\times}10^{-4}$ & $3{\times}10^{-4}$ & $3{\times}10^{-4\,\dagger}$ & $3{\times}10^{-4}$ \\
$\widehat V_k$ update epochs            & $10$ & $10$ & $10$ & $10$ & $5$ & $5$ & $10$ & $10$ \\
$\widehat Q$ ensemble heads $N$         & $2$ & $2$ & $2$ & $2$ & $2$ & $2$ & $10$ & $2$ \\
\midrule
\multicolumn{9}{l}{\textit{Networks}} \\
Actor MLP                               & $[512]^{3}$ & $[512]^{3}$ & $[1024]^{3}$ & $[1024]^{3}$ & $[1024]^{7}$ & $[1024]^{7}$ & $[256]^{3}$ & $[512]^{3}$ \\
Actor time embedding                    & $16$ & $16$ & $32$ & $32$ & $32$ & $32$ & $16$ & $16$ \\
$\widehat V$ MLP                        & $[256]^{3}$ & $[256]^{3}$ & $[256]^{3}$ & $[256]^{3}$ & $[512]^{3}$ & $[512]^{3}$ & $[256]^{3}$ & $[256]^{3}$ \\
$\widehat Q$ MLP                        & $[256]^{3}$ & $[256]^{3}$ & $[256]^{3}$ & $[256]^{3}$ & $[512]^{3}$ & $[512]^{3}$ & $[256]^{3}$ & $[256]^{3}$ \\
$\widehat V_k$ MLP                      & $[256]^{3}$ & $[256]^{3}$ & $[256]^{3}$ & $[256]^{3}$ & $[512]^{3}$ & $[512]^{3}$ & $[256]^{3}$ & $[256]^{3}$ \\
\bottomrule
\end{tabular}}
\\[2pt]
\footnotesize $^{\ast}$Environment steps collected per iteration are this value
times $T_a$ times $N_{\mathrm{env}}$. \quad
$^{\dagger}$\textsc{Partial} uses $10^{-3}$; \textsc{Complete} and \textsc{Mixed}
use $3\times10^{-4}$.
\end{table*}

\end{document}